\documentclass[nohyperref, 11pt]{article}

\usepackage{geometry}
\usepackage[utf8]{inputenc} 
\usepackage[T1]{fontenc}    
\usepackage{times}
\usepackage{url}            
\usepackage{booktabs}       
\usepackage{amsfonts}       
\usepackage{microtype}      
\usepackage{xcolor}         

\usepackage{algorithm}
\usepackage{algpseudocode}
\usepackage{amsmath}
\usepackage{amssymb}
\usepackage{amsthm}
\usepackage{bbm}
\usepackage{caption}
\usepackage{color}
\usepackage{enumitem}
\usepackage{fancyhdr}
\usepackage{geometry}
\usepackage{graphicx}
\usepackage[hidelinks]{hyperref}
\usepackage{listings}
\usepackage{makecell}
\usepackage{mathtools}
\usepackage{multirow}
\usepackage[square,numbers]{natbib}
\usepackage{setspace}
\usepackage{subcaption}
\usepackage[most]{tcolorbox}
\usepackage{xspace}
\usepackage{nicefrac}
\usepackage[mathscr]{eucal}
\PassOptionsToPackage{capitalise,noabbrev}{cleveref}

\newcommand{\symfootnote}[1]{%
\let\oldthefootnote=\thefootnote%
\stepcounter{mpfootnote}%
\addtocounter{footnote}{-1}%
\renewcommand{\thefootnote}{\fnsymbol{mpfootnote}}%
\footnote{#1}%
\let\thefootnote=\oldthefootnote%
}

\title{Robustness to Unbounded Smoothness of Generalized SignSGD}

\author{
Michael Crawshaw\symfootnote{*Authors in alphabetical order.}\\ George Mason University \\ {mcrawsha@gmu.edu}
\and
Mingrui Liu\footnotemark[1]\\ George Mason University \\
{mingruil@gmu.edu} 
\and
Francesco Orabona\footnotemark[1] \\ Boston University \\
{francesco@orabona.com}
\and
Wei Zhang\footnotemark[1]\\ IBM T.~J.~Watson Research Center\\ {weiz@us.ibm.com}
\and
Zhenxun Zhuang\footnotemark[1]\\ Boston University \\ {zxzhuang@bu.edu}
}

\date{}

\newcommand{\ba}{\boldsymbol{a}}
\newcommand{\bb}{\boldsymbol{b}}

\newcommand{\bd}{\boldsymbol{d}}

\newcommand{\bm}{\boldsymbol{m}}

\newcommand{\bg}{\boldsymbol{g}}
\newcommand{\bx}{\boldsymbol{x}}
\newcommand{\bu}{\boldsymbol{u}}
\newcommand{\by}{\boldsymbol{y}}

\newcommand{\bv}{\boldsymbol{v}}

\newcommand{\boldepsilon}{\boldsymbol{\epsilon}}

\newcommand{\E}{\mathbb{E}}
\newcommand{\R}{\mathbb{R}}

\newtheorem{theorem}{Theorem}
\newtheorem{lemma}{Lemma}

\newtheorem{assumption}{Assumption}

\mathtoolsset{showonlyrefs}

\allowdisplaybreaks

\newcommand{\VecLOneNorm}[1]{\|\boldsymbol{#1}\|_1}
\newcommand{\VecLInftyNorm}[1]{\|\boldsymbol{#1}\|_{\infty}}

\newcommand{\PartialDerivative}[2]{\frac{\partial F}{\partial x_{#2}}(\bx_{#1})}
\newcommand{\PartialDerivativeGeneral}[2]{\frac{\partial F}{\partial x_{#2}}(#1)}

\newcommand{\StocGradient}[2]{g_{#1, #2}}

\newtheorem{customcounttheorem}{Theorem}
\newenvironment{addcustomcounttheorem}[1]
  {\renewcommand\thecustomcounttheorem{#1}\customcounttheorem}
  {\endcustomcounttheorem}

\newtheorem{customcountlemma}{Lemma}
\newenvironment{addcustomcountlemma}[1]
  {\renewcommand\thecustomcountlemma{#1}\customcountlemma}
  {\endcustomcountlemma}

\newcommand{\BoldLZeroLOne}{\ensuremath{(\boldsymbol{L_0} ,\boldsymbol{L_1})}\xspace}

\newcommand{\PNorm}[1]{\|#1\|_2}
\newcommand{\PNormDimension}{\sqrt{d}}

\begin{document}

\maketitle

\begin{abstract}
Traditional analyses in non-convex optimization typically rely on the smoothness assumption, namely requiring the gradients to be Lipschitz. However, recent evidence shows that this smoothness condition does not capture the properties of some deep learning objective functions, including the ones involving Recurrent Neural Networks and LSTMs. Instead, they satisfy a much more relaxed condition, with potentially unbounded smoothness. Under this relaxed assumption, it has been theoretically and empirically shown that the gradient-clipped SGD has an advantage over the vanilla one. In this paper, we show that clipping is not indispensable for Adam-type algorithms in tackling such scenarios: we theoretically prove that a generalized SignSGD algorithm can obtain similar convergence rates as SGD with clipping but does not need explicit clipping at all. This family of algorithms on one end recovers SignSGD and on the other end closely resembles the popular Adam algorithm. Our analysis underlines the critical role that momentum plays in analyzing SignSGD-type and Adam-type algorithms: it not only reduces the effects of noise, thus removing the need for large mini-batch in previous analyses of SignSGD-type algorithms, but it also substantially reduces the effects of unbounded smoothness and gradient norms. We also compare these algorithms with popular optimizers on a set of deep learning tasks, observing that we can match the performance of Adam while beating the others.
\end{abstract}

\section{Introduction}
\label{sec:intro}

Recent years have witnessed a surge in non-convex machine learning models, with a focus on deep neural networks~\cite{lecun2015deep}. DNNs have achieved tremendous progress in a variety of tasks, including computer vision~\cite{krizhevsky2012imagenet, he2016deep, kamp2018efficient}, natural language processing~\cite{devlin2018bert, VaswaniSPUJGKP17}, and a lot more. Despite their huge empirical success, the theoretical analyses of non-convex optimization~\cite{Jain2017NonconvexOF} prove to be fundamentally more challenging than the established convex optimization theory~\cite{boyd2004convex}. Among the numerous literature, many of them assume smoothness of the objective function, namely requiring the gradients to be Lipschitz. Under this scenario, past works have succeeded in proving the convergence rates for a number of algorithms, e.g., Stochastic Gradient Descent~\cite{ghadimi2013stochastic}, AdaGrad~\cite{WardWB19,LiO19}, and STORM~\cite{CutkoskyO19,cutkosky2020momentum}.

Nevertheless, it was recently observed that the smoothness assumption does not capture the training of LSTMs~\cite{hochreiter1997long}: the Hessian can grow with the size of the gradients~\cite{zhang2019gradient}. Inspired by this, Zhang et al.~\cite{zhang2019gradient} proposed a relaxed smoothness assumption, named $(L_0, L_1)$ smoothness:
\begin{equation}
    \|\nabla^2 F(\bx)\| \le L_0 + L_1\|\nabla F(\bx)\|~.\label{eq:global_l0l1}
\end{equation}
They also showed that the well-known gradient clipping technique can ensure SGD's convergence in such scenarios. Later, their results were improved to show that SGD with clipping can be made unaffected by the $L_1$ in~\eqref{eq:global_l0l1} and is able to recover the optimal convergence rate of SGD under the original smoothness setting~\cite{zhang2020improved, Robustness21Jin}.

Nevertheless, the $(L_0, L_1)$ condition has not yet been empirically verified beyond LSTMs. Therefore, our \textbf{first contribution} lies in studying the applicability and generalization of the $(L_0, L_1)$ condition. In particular, we have empirically verified that the popular Transformer~\cite{VaswaniSPUJGKP17} model also seems to satisfy this assumption, see Figure~\ref{fig:global_l0l1}. Yet, we noticed that different coordinates, especially when they are in different layers of the model, exhibit very distinct $L_0$ and $L_1$ values as shown in Figure~\ref{fig:transformer_l0l1_coordinate_wise}. Hence, we propose to refine the $(L_0, L_1)$ assumption in~\eqref{eq:global_l0l1} to a coordinate-wise version (Assumption~\ref{asp:l0l1_coordinate}) and consider this to better capture the loss surface when training deep neural networks like Transformers.

\begin{figure}[t]
\begin{minipage}{.48\textwidth}
        \centering
        \begin{minipage}{0.48\linewidth}
            \includegraphics[width=\linewidth]{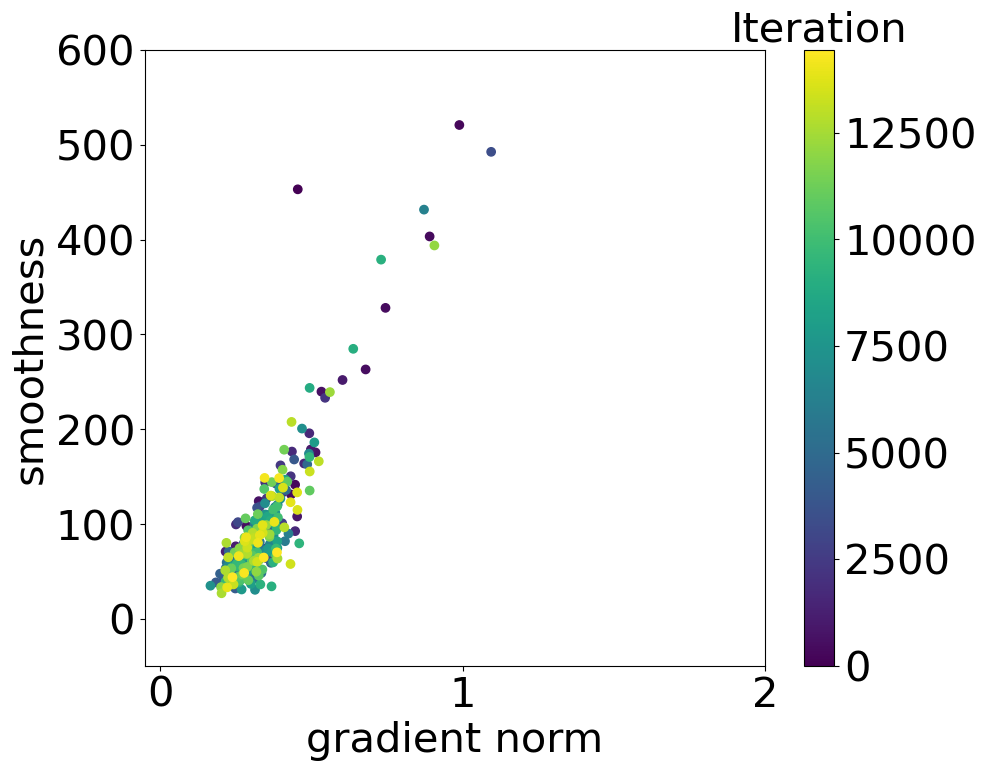}
            {\centerline{(a) Wikitext-2}}
        \end{minipage}
        \hfill
        \begin{minipage}{0.48\linewidth}
            \includegraphics[width=\linewidth]{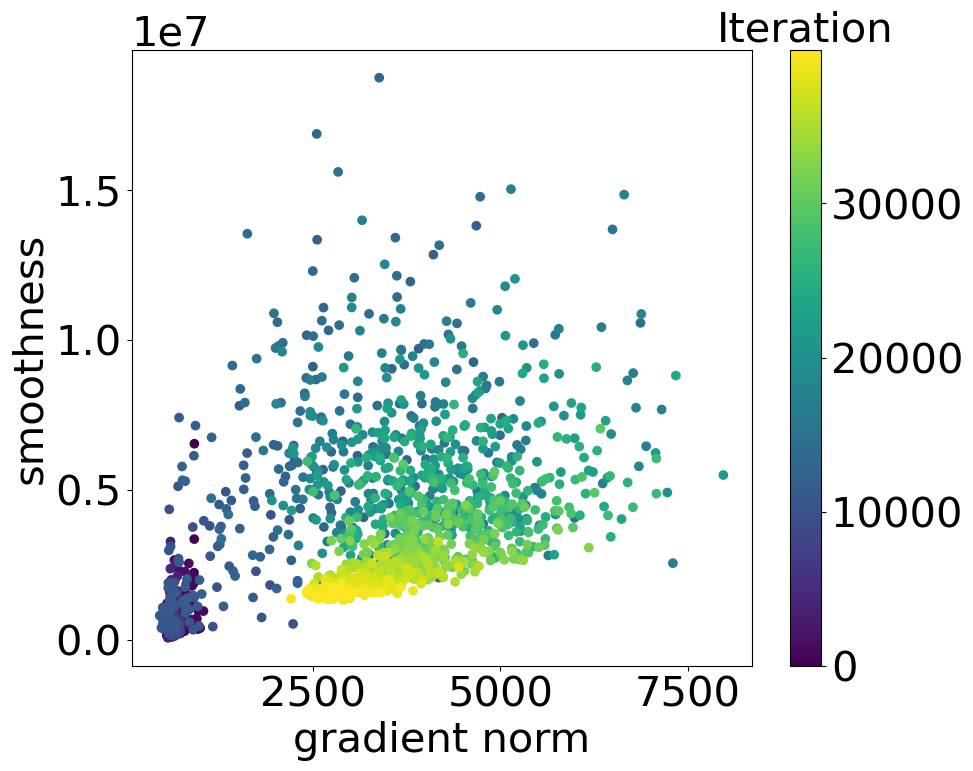}
            {\centerline{(b) WMT'16 de-en}}
        \end{minipage}
        \captionof{figure}{Local gradient Lipschitz constant vs.~Gradient norm on training (a) a $2$-layer Transformer Encoder on Wikitext-2 (b) a $6$-layer Transformer on WMT'16 Multimodal Machine Translation de-en dataset. The colorbar indicates \#Iterations in training. Details in Section~\ref{ssec:transformer_l0l1}.}
        \label{fig:global_l0l1}
    \end{minipage}
    \hspace{.02\textwidth}
    \begin{minipage}{.48\textwidth}
        \centering
        \includegraphics[width=0.85\textwidth]{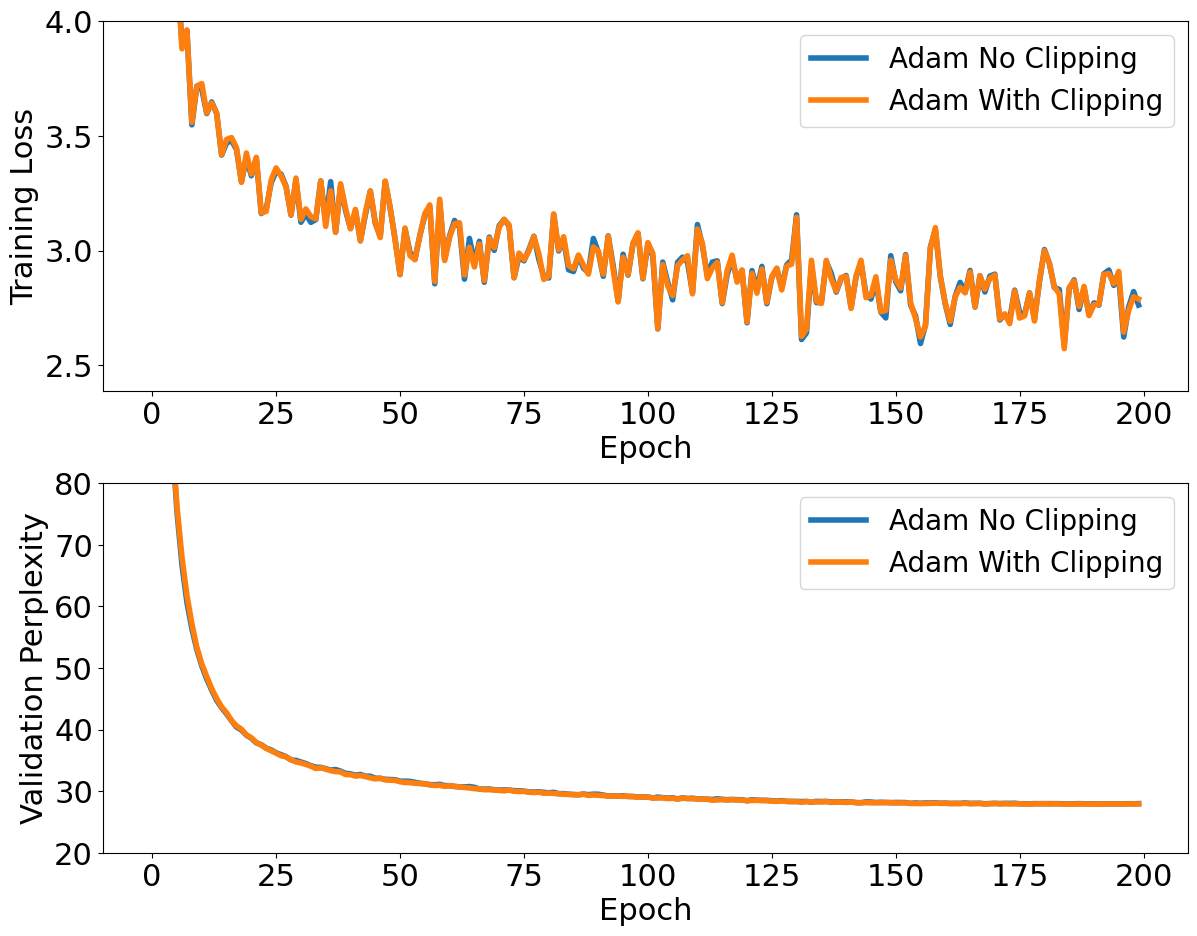}
        \captionof{figure}{Training GPT-2 on Wikitext-103 using Adam with or without gradient clipping.}
        \label{fig:gpt2_adam_clipping}
    \end{minipage}
\end{figure}

\begin{figure}[t]
     \centering
     \begin{subfigure}[b]{0.24\textwidth}
         \centering
         \includegraphics[width=\textwidth]{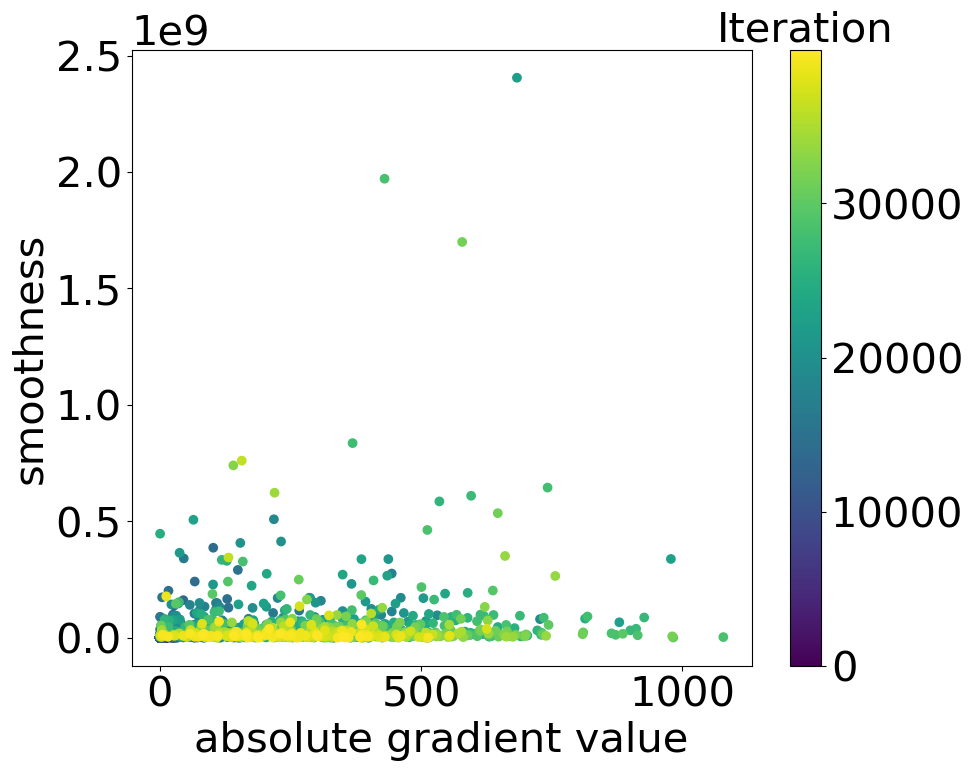}
         \caption{Encoder First Layer}
         \label{fig:enc_first}
     \end{subfigure}
     \hfill
     \begin{subfigure}[b]{0.24\textwidth}
         \centering
         \includegraphics[width=\textwidth]{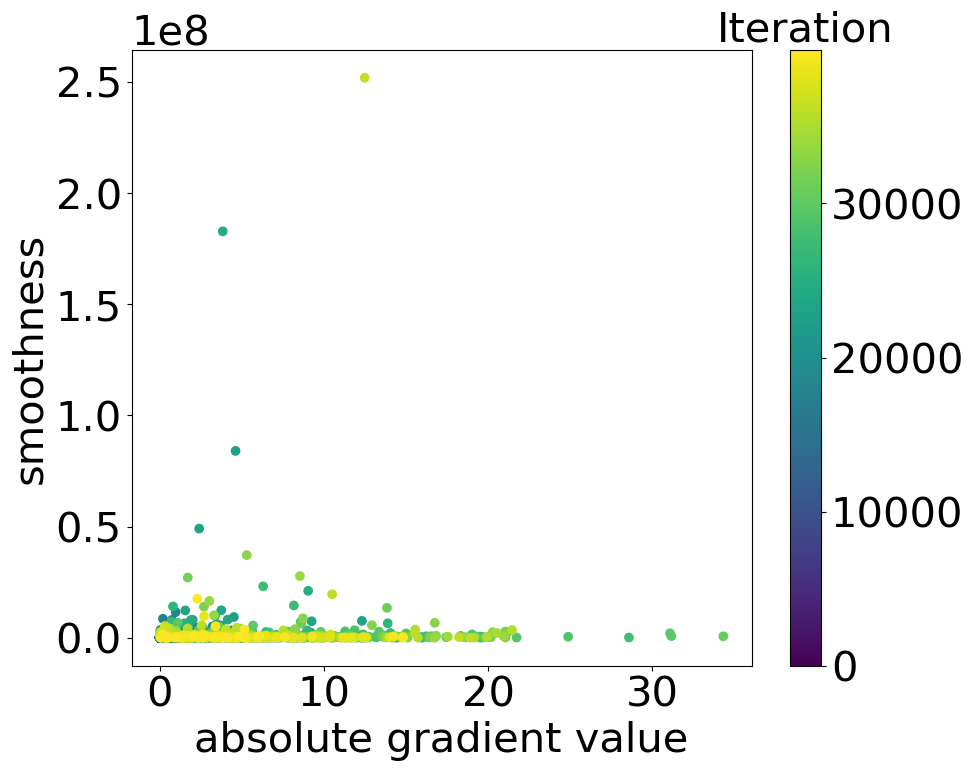}
         \caption{Encoder Last Layer}
         \label{fig:enc_last}
     \end{subfigure}
     \hfill
     \begin{subfigure}[b]{0.24\textwidth}
         \centering
         \includegraphics[width=\textwidth]{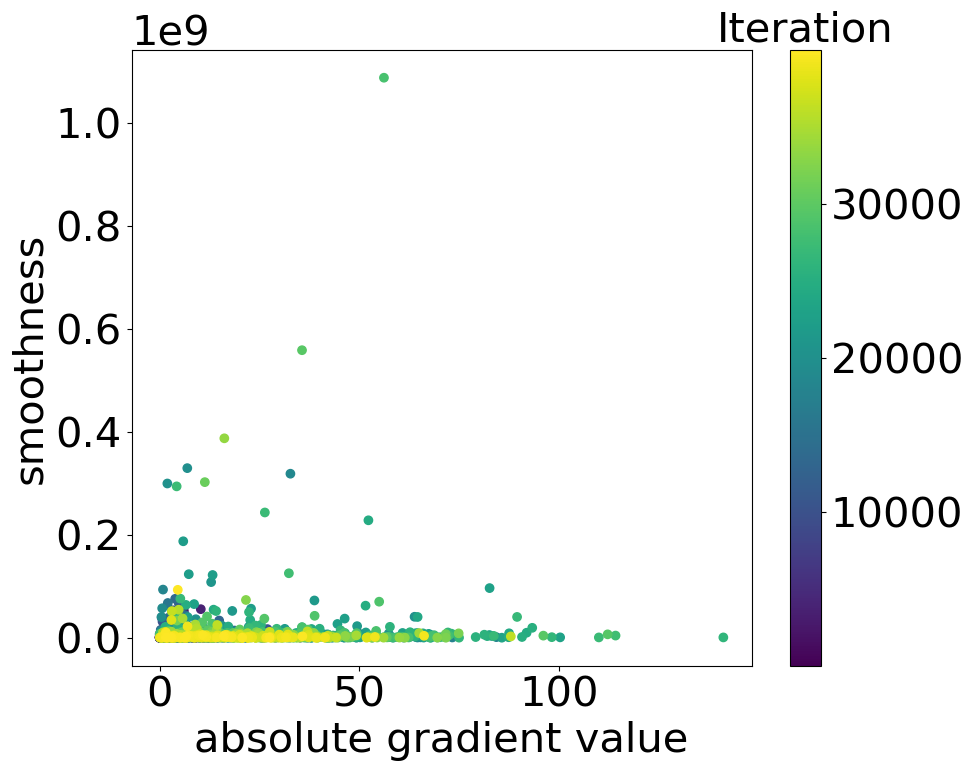}
         \caption{Decoder Second Layer}
         \label{fig:dec_second}
     \end{subfigure}
     \hfill
     \begin{subfigure}[b]{0.24\textwidth}
         \centering
         \includegraphics[width=\textwidth]{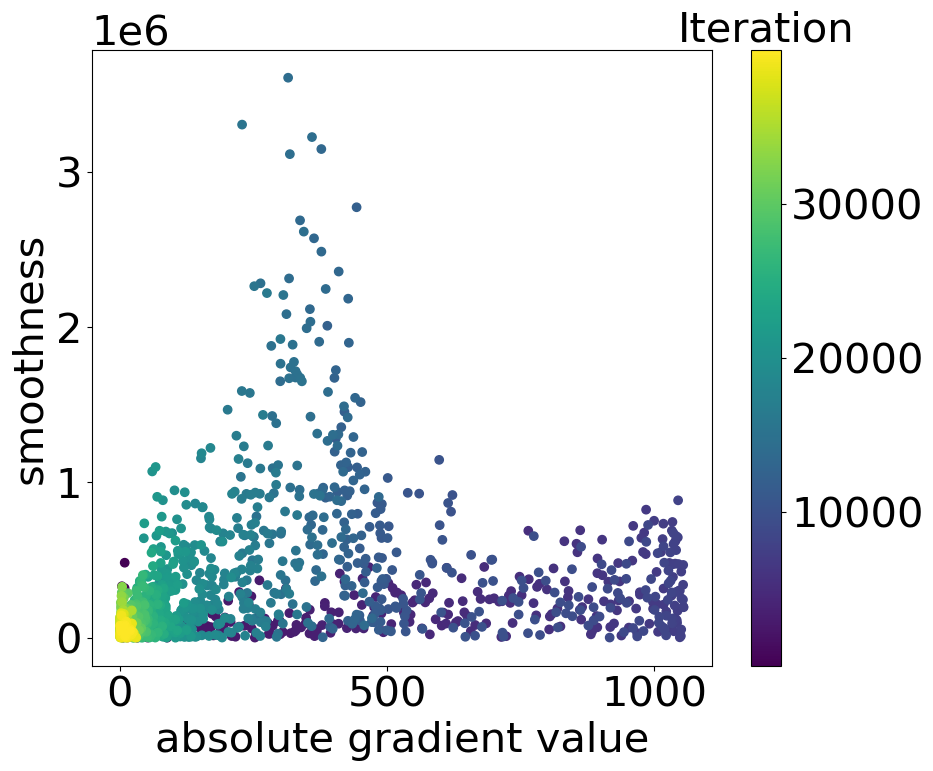}
         \caption{Decoder Last Layer}
         \label{fig:dec_last}
     \end{subfigure}
    \caption{Local gradient Lipschitz constant vs.~absolute gradient value on training a Transformer on WMT'16 Multimodal Translation de-en dataset. Each figure represents a randomly picked coordinate in corresponding layers. The colorbar indicates \#Iterations during training. Details in Section~\ref{ssec:transformer_l0l1}.}
    \label{fig:transformer_l0l1_coordinate_wise}
\end{figure}

Given that we assume (a generalization) of the $(L_0,L_1)$ assumption, it would be natural to use some clipping procedure. However, we found out that the use of clipping on Adam~\cite{KingmaB14}, while carried out in common practice~\cite[e.g.,][]{gpt2_naacl}, \emph{has no effect on the training and testing performance on optimizing a large transformer model as shown in Figure~\ref{fig:gpt2_adam_clipping}.}
In retrospect, this might not be surprising: It is known that Adam has an implicit clipping behavior due to the normalization by the estimated second moment of the gradients. Indeed, Adam can be interpreted as a variant of SignSGD~\cite{balles2018dissecting}.

Inspired by this, our \textbf{second contribution} is to propose and analyze a \emph{generalized SignSGD} algorithm under the relaxed smoothness assumption. It is parameterized in such a way that it on one end recovers SignSGD while on the other end closely resembles Adam. Apart from the convergence rates, we also located the critical role the momentum plays in analyzing Adam-type algorithms: it not only reduces the effects of noise but also gives an exponential decaying effect on the unbounded gradient norms and smoothness. This can partly explain the phenomenon that clipping does not help Adam.

The structure of this paper is as follows. Section~\ref{sec:related} discusses related works and how our paper builds upon and distinguishes from them. The settings and assumptions are carried out in Section~\ref{sec:settings}. We will introduce formally the generalized SignSGD algorithm and its analysis in Section~\ref{sec:theory}, with a detailed discussion on the bounds and the role of momentum. The experimental results are shown in Section~\ref{sec:experiments}, comparing our algorithm with some popular competitors in deep learning tasks. Finally, we draw some conclusions and discuss the limitations of our work in Section~\ref{sec:conclusion}.

\textbf{Notations} We will use $[d]$ to denote the sequence $[1,2,\ldots,d]$ and use bold letters to represent vectors, e.g., $\bu\in\R^d$.
The $j$-th coordinate of a vector $\bu$ is $u_j$.
Throughout this paper, we study the Euclidean space $\R^d$ with the inner product $\langle\cdot, \cdot\rangle$.
$\E[\bu]$ means the expectation with respect to the underlying probability distribution of a random variable $\bu$, and $\E_t[\bu]$ is the conditional expectation of $\bu$ conditioned on the past of time $t$.
The gradient of $F$ at $\bx$ is denoted by $\nabla F(\bx)$.
We use $\mathbb{I}(\cdot)$ to denote the indicator function, $\|\bu\|_p$ to denote the $p$-norm: $\|\bu\|_p:=\left(\sum^d_{j=1}|u_j|^p\right)^{1/p}$ and $\|\bu\|_{\infty}$ the maximum norm: $\|\bu\|_{\infty}:=\max\{|u_1|,\ldots,|u_d|\}$. We also denote by $\sum^j_{k=i} x_k = 0$ when $i > j$.

\section{Related Works}
\label{sec:related}

\textbf{Adaptive Gradient Methods} Adaptive gradient methods~\cite{mcmahan2010adaptive,duchi2011adaptive,KingmaB14,hinton2012deep,reddi2019convergence} are popular optimizers for training deep neural networks. The traditional analysis of adaptive gradient methods is providing regret bounds under the online convex optimization framework~\cite{duchi2011adaptive,KingmaB14,reddi2019convergence}. Recently, there are some analysis of adaptive gradient methods for nonconvex smooth functions~\cite{chen2018universal,chen2018convergence,zaheer2018adaptive,defossez2020simple,zou2019sufficient}. Zou et al.~\cite{ZouCLG21} introduces an intriguing connection between Adam~\cite{KingmaB14} and SignGD~\cite{bernstein2018signsgd} when training a two-layer neural network in the deterministic setting, where SignGD is an algorithm following the negative gradient sign direction to perform the update. However, these works cannot be directly extended to nonconvex functions with unbounded smoothness in the stochastic setting. To the best of our knowledge, this work is the first one establishing guarantees for coordinate-wise type optimizers like generalized SignSGD as well as Adam-type updates under a relaxed smoothness condition. 

\noindent\textbf{Gradient Clipping}  The algorithm and analysis of gradient clipping can be traced back to~\cite{alber1998projected,shor2012minimization,ermoliev1988stochastic} under the assumption that the function is convex and rapidly growing. Hazan et al.~\cite{hazan2015beyond} considered gradient clipping in quasi-convex optimization. Mai and Johansson~\cite{mai2021stability} showed the stability and convergence of stochastic gradient clipping algorithms for convex problems without the smoothness condition. Gradient clipping is a standard technique in training deep neural networks~\cite{pascanu2012understanding,pascanu2013difficulty} such as RNNs and LSTMs. The theoretical analysis of gradient clipping for nonconvex models is pioneered by~\cite{zhang2019gradient}, in which the authors analyzed the convergence of gradient clipping under the relaxed smoothness assumption rather than the standard smoothness assumption. Zhang et al.~\cite{zhang2020improved} further improved the convergence rate bound under the same assumption as in~\cite{zhang2019gradient}. Gradient clipping is also used when there is a heavy tail noise in the stochastic gradient to establish high probability convergence rates~\cite{CutkoskyM21,gorbunov2020stochastic,zhang2020adaptive}. Cutkosky and Mehta~\cite{cutkosky2020momentum} proved that normalized momentum improves normalized SGD under a second-order smoothness condition. 
A close algorithm is the one in \cite{Robustness21Jin} which employs gradient normalization, momentum, and no gradient clipping to tackle the $(L_0, L_1)$ condition~\eqref{eq:global_l0l1} and control noise. Yet, their algorithm normalizes each coordinate with the same scale unlike popular optimizers such as Adam~\cite{KingmaB14}. Moreover, we observe empirically that normalized SGD with momentum performs worse than Adam. Motivated by this, we propose a coordinate-wise optimization algorithm which requires new analysis tools compared with~\cite{Robustness21Jin}.

\noindent\textbf{Employ $\bm_t^2$ to compute $\bv_t$ in Adam} Designed to combine the advantages of Adagrad~\cite{duchi2011adaptive} and RMSProp~\cite{tieleman2012lecture}, the update of Adam~\cite{KingmaB14} employs the ratio between the exponential moving average of the stochastic gradient ($\bm_t$) and the exponential moving average of the squared stochastic gradient ($\bv_t$). Many variants of Adam have been proposed ever since. Among them, one idea is to use $\bm_t^2$ to compute $\bv_t$ instead of $\bg_t^2$. The intuition is that $\bm_t$ represents a better update direction than $\bg_t$ and can thus better capture the second-moment information. Reddi et al.~\cite{reddi2021adaptive} adopted this change to prove the convergence of Adam in a federated learning setting; yet, they only consider the smooth setting and require a large $\epsilon$ to obtain convergence in contrast to the original Adam. Later, Wang et al.~\cite{WangKQWXZF21} explored this idea in more detail, but their analyses are still restricted to the smooth setting.

\section{Settings and Preliminaries}
\label{sec:settings}

In this paper, we focus on the following stochastic optimization problem:
\begin{equation}
    \min_{\bx\in\R^d} F(x) := \E_{\xi\sim\mathcal{D}}[f(\bx, \xi)],
\end{equation}
where $\xi$ is a random variable representing a randomly selected data sample or random noise following an unknown distribution $\mathcal{D}$.
We will use the following assumptions.
\begin{assumption}
\label{asp:objective_func}
$F: \R^d\rightarrow \R$ is differentiable and bounded from below with infimum $F^*$.
\end{assumption}
\begin{assumption}
\label{asp:l0l1_coordinate}
We say that a twice differentiable function $F(\bx)$ is $(\boldsymbol{L_0}, \boldsymbol{L_1})$-smooth coordinate-wisely, if for any $\bx,\by \in \R^d$ for which $\PNorm{\bx - \by} \le \frac{1}{\VecLInftyNorm{L_1}}$, we have for any $j\in[d]$ that
\begin{equation}
\left|\frac{\partial F}{\partial x_j}(\by) - \frac{\partial F}{\partial x_j}(\bx) \right|
\leq
\left(\frac{L_{0,j}}{\sqrt{d}} + L_{1,j}\left|\frac{\partial F}{\partial x_j}(\bx)\right|\right)\PNorm{\by - \bx}~.\label{eq:coordinate_wise_l0l1}
\end{equation}
We will denote $\boldsymbol{L_0}:=[L_{0,1}, L_{0,2},\ldots,L_{0,d}]^T$ and $\boldsymbol{L_1}:=[L_{1,1}, L_{1,2},\ldots,L_{1,d}]^T$.
\end{assumption}

The original $(L_0, L_1)$ smoothness assumption~\eqref{eq:global_l0l1} in~\cite{zhang2019gradient} was proposed as a generalization of the more common smoothness assumption, which says that the gradient should be Lipschitz. Indeed, when $L_{1,j}$ are zero, we recover the smoothness assumption. In contrast, when $L_{1,j}$ are non-zero, the smoothness of the function is potentially \emph{unbounded}.
Yet, \cite{zhang2019gradient} works with norms and applies to the global scale, while ours is more fine-grained and applies to each coordinate separately. One motivation for this assumption comes from~\cite[Remark 2.3,][]{zhang2020improved} where they noted that~\eqref{eq:global_l0l1} can be relaxed to an assumption on gradient differences: there exists $K_0, K_1 > 0$ such that for all $\bx, \by \in \R^d$ with $\|\bx - \by\|_2 \le \frac{1}{K_1}$,
\begin{equation}
\label{eq:l0l1_global_gradient}
    \|\nabla F(\bx) - \nabla F(\by)\|_2
    \le
    (K_0 + K_1\|\nabla F(\bx)\|_2)\|\bx - \by\|_2
\end{equation}

Indeed, our Assumption~\ref{asp:l0l1_coordinate} implies~\eqref{eq:l0l1_global_gradient} when $L_{0,j} = L_0$ and $L_{1,j} = L_1$ for all $j\in[d]$, up to constants (See Lemma~\ref{lem:l0l1_recover} in the Appendix). Note that the $\frac{1}{\sqrt{d}}$ factor in ours is exactly for easy comparison with~\eqref{eq:l0l1_global_gradient}. The reason we turn to the current coordinate-wise version is that we observed a vast variance across different layers in training Transformer models: \eqref{eq:global_l0l1} is still true globally (Figure~\ref{fig:global_l0l1}), but each layer or even each coordinate satisfies has a very different $(L_0, L_1)$ pair (Figure~\ref{fig:transformer_l0l1_coordinate_wise}). The smoothness assumption has been generalized in orthogonal directions in other work~\cite{RichtrikT14, bernstein2018signsgd, khaled2020better}. 

One merit of Assumption~\ref{asp:l0l1_coordinate} is that it gives us the following descent lemma.
\begin{lemma}
\label{lem:desent_ineq_l0l1_coordinatewise}
Let $F$ be \BoldLZeroLOne-smooth coordinate-wisely. Then, for any $\bx, \by\in\R^d$ for which $\PNorm{\bx - \by} \le \frac{1}{\VecLInftyNorm{L_1}}$, we have
\begin{equation}
    F(\by)
    \leq
    F(\bx) + \left\langle\nabla F(\bx), \by-\bx\right\rangle + \sum^d_{j=1}\frac{ \left(\frac{L_{0,j}}{\sqrt{d}}+ L_{1,j}\left|\PartialDerivativeGeneral{\bx}{j}\right|\right)\PNorm{\by - \bx}}{2}|y_j-x_j|~.
\end{equation}
\end{lemma}

Our last assumption is common in the literature studying the $(L_0, L_1)$ smooth condition~\cite{zhang2019gradient, zhang2020improved}.
\begin{assumption}
\label{asp:noise_as_bounded}
For each $j \in [d]$, there exists $\sigma_j > 0$ such
that for all $\bx \in \R^d$ and $\xi\sim\mathcal{D}$, the noise satisfies $\left|[\nabla f(\bx, \xi)]_j - \PartialDerivativeGeneral{\bx}{j}\right| \le \sigma_j$ with probability $1$. We will denote $\boldsymbol{\sigma}:=[\sigma_{1}, \sigma_{2},\ldots,\sigma_{d}]^T$.
\end{assumption}

\section{A Generalized SignSGD Algorithm}
\label{sec:theory}

\begin{algorithm}[H]
    \caption{Generalized SignSGD \emph{(All operations on vectors are element-wise.)}}\label{alg:generalized_signsgd}
	\begin{algorithmic}[1]
	    \State Inputs: $\bx_1$, $\beta_1$, $\beta_2$, $\eta$
	    \State $\bm_0 = 0$, $\bv_0 = 0$
	    \For {$t = 1, \cdots, T$}
            \State Compute an unbiased estimate $\nabla f(\bx_t, \xi_t)$ of $\nabla F(\bx_{t})$, denoted as $\bg_t$
            \State $\bm_t = \beta_1 \bm_{t-1} + (1 - \beta_1) \bg_t$
            \State $\bv_t = \beta_2 \bv_{t-1} + (1 - \beta_2) \bm_t^2$
            \State $\bx_{t+1} = \bx_{t} - \eta \frac{\bm_t}{\sqrt{\bv_t}}$
		\EndFor
	\end{algorithmic}
\end{algorithm}

In this section, we present in Algorithm~\ref{alg:generalized_signsgd} a generalized SignSGD algorithm. This algorithm encompasses a variety of optimization algorithms.

At first sight, it seems very similar to Adam. Indeed, if we employ $\bg_t^2$ in computing $\bv_t$ instead of $\bm_t^2$, then it is exactly Adam, except for the bias correction terms. We would like to clarify that the idea of this change has been explored before, as detailed in Section~\ref{sec:related}. In this paper, the motivation for adopting this idea comes from the known effect of momentum on reducing the influence of noises~\cite{cutkosky2020momentum}. Indeed, in our analysis the difference between $\bm_t$ and $\nabla F(\bx_t)$ is much more controllable than between $\bg_t$ and $\nabla F(\bx_t)$. Thus, we consider employing $\bm_t$ in computing $\bv_t$ a better choice.

On the other end, the careful reader might observe that Algorithm~\ref{alg:generalized_signsgd} recovers the SignSGD with Momentum algorithm, also called SIGNUM in~\cite{bernstein2018signsgd}, when setting $\beta_2 = 0$. Sign-based algorithms are naturally suited to distributed learning~\cite{li2014scaling} and the idea dated back to at least RPROP~\cite{riedmiller1993direct}. The convergence to a stationary point (with $\ell_1$ norm) under a coordinate-wise smoothness condition has been established for SignSGD with/without the momentum in~\cite{bernstein2018signsgd} though they necessitate large mini-batches to control the variance of the noise. Yet, we are more interested in their property of the update size being bounded without the need for explicit clipping.

Note that both SignSGD and Adam are good candidates for optimization algorithms whose update must be bounded on functions that satisfy the \BoldLZeroLOne condition.
Indeed, SignSGD can be seen as an extreme form of gradient clipping. On the other hand, as said in the introduction, Adam does not seem to require gradient clipping at all when used to train the large Transformer model in Figure~\ref{fig:gpt2_adam_clipping}. 

Hence, we expect our algorithm, a generalization of SignSGD and a close resemblance to Adam, can enjoy the merits of both and be robust to the unbounded smoothness in the \BoldLZeroLOne scenario. In the next section, we will formalize this claim by presenting the theoretical analysis of Algorithm~\ref{alg:generalized_signsgd}.

\subsection{Theoretical Convergence Analysis}
\label{ssec:analyses}
\begin{theorem}
\label{thm:generalized_signsgd}
Under Assumptions~\ref{asp:objective_func},~\ref{asp:l0l1_coordinate}, and~\ref{asp:noise_as_bounded}, assume $M_j := \sup \left\{ \left|\PartialDerivativeGeneral{\bx}{j}\right| : F(\bx) \leq F(\bx_1)\right\}$ is finite for each $j\in[d]$, let $\Delta$ be any upper bound on 
$F(\bx_1) - F^*$,
$\alpha = \min\left(\frac{\sqrt{\VecLOneNorm{L_0}}\sqrt{\Delta}}{\VecLOneNorm{\sigma}\sqrt{T}}, 1\right)$, $\beta_1 = 1 - \alpha$, $\frac{\sqrt{\beta_2}}{\beta_1} < 1$,
$\rho = 1-\frac{\sqrt{\beta_2}}{\beta_1}$,
$\eta = \frac{\sqrt{\Delta\alpha}}{\sqrt{\VecLOneNorm{L_0}}\sqrt{T}}$,
for $T\ge\max\left(\frac{100d\Delta\VecLInftyNorm{L_1}^2}{(1-\beta_2)\rho^2\VecLOneNorm{L_0}}, \frac{10000d^2\Delta\VecLOneNorm{\sigma}^2\VecLInftyNorm{L_1}^4}{(1-\beta_2)^2\rho^4\VecLOneNorm{L_0}^3}\right)$, Algorithm~\ref{alg:generalized_signsgd} guarantees, with probability at least $1 - \delta$, that
\begin{align}
    \min_{t\in[T]} \ \|\nabla F(\bx_t)\|_1
    =
    &\mathcal{O}\left(\frac{\sqrt{\log(dT/\delta)}\VecLOneNorm{L_0}^{1/4}\Delta^{1/4}\VecLOneNorm{\sigma}^{1/2}}{\rho\sqrt{1-\beta_2}T^{1/4}} + \frac{\log(dT/\delta)\sqrt{\VecLOneNorm{L_0}\Delta}}{\rho\sqrt{T}}\right)\\
    &+ \mathcal{O}\left(\frac{\VecLOneNorm{M} + \VecLOneNorm{\sigma}}{\rho}\exp\left(-\frac{\sqrt{1-\beta_2}\VecLOneNorm{L_0}^{3/4}}{\PNormDimension\VecLInftyNorm{L_1}\VecLOneNorm{\sigma}^{1/2}\Delta^{1/4}}T^{1/4}\right)
    +\frac{\|\nabla F(\bx_1)\|_1}{T} \right)~.
\end{align}

Furthermore, for the case when $\beta_2 = 0$, we have the following refined guarantee:
\begin{align}
\min_{t\in[T]} \ \|\nabla F(\bx_t)\|_1
=
&\mathcal{O}\left(\frac{\sqrt{\log(dT/\delta)}\VecLOneNorm{L_0}^{1/4}\Delta^{1/4}\VecLOneNorm{\sigma}^{1/2}}{T^{1/4}}
+ \frac{\log(dT/\delta)\sqrt{\VecLOneNorm{L_0}\Delta}}{\sqrt{T}}\right)\\
& + \mathcal{O}\left(\frac{\|\nabla F(\bx_1)\|_1}{\sqrt{T}}\left(\frac{1}{\sqrt{T}} + \frac{\VecLOneNorm{\sigma}}{\sqrt{\VecLOneNorm{L_0}\Delta}}\right)
+ \frac{\VecLOneNorm{\sigma}}{T}\right)~.
\end{align}
\end{theorem}
Here, $M_j$ denotes the maximum absolute value of the partial derivative of $F$ for coordinate $j$ among the sub-level set of $F(\bx_1)$, namely any point $\bx$ with $F(\bx) \le F(\bx_1)$. In other words, we assume gradients to be bounded in the sub-level set of $F(\bx_1)$; yet, we do not make any restriction on gradients outside of this set. We believe this is not a strong assumption, for example, when the sub-level set of $F(\bx_1)$ is bounded, then by the assumed continuity of gradients it trivially holds. Also, we just require an upper bound and it can even be exponentially large as we have an exponentially decaying coefficient to counteract it: notice how the term $\VecLOneNorm{M}$ is multiplied by a term that decays exponentially with $T$. Better still, when $\beta_2 = 0$, we no longer even need this assumption and the algorithm is entirely free of the influence of $\VecLOneNorm{M}$. To see why this is good, we show a refined lower bound of Gradient Descent under the relaxed smoothness scenario below which is originally in~\cite{zhang2019gradient}.
\begin{theorem}
\label{thm:lower_bound_gd_fixed}
Fix $\epsilon > 0, L_0>0, L_1>0, M\geq \max(\frac{L_0}{L_1},\epsilon)$, and $x_0 \in \mathbb{R}$. Pick any constant learning rate $\eta$ for GD, with the knowledge of the above constants. Then, there exists a 1-d $(L_0,
L_1)$-smooth function, bounded from below by $f^*$ (finite), and such that $ \sup \{ |
f'(x)| : f(x) \leq f(x_0)\} \leq M$ on which the number of iterations $T$ of GD with learning rate $\eta$ to guarantee $| f'(x_T)|< \epsilon$ is at least
\begin{equation}
    \frac{ M L_1 (f(x_0) - f^*-\frac{15 \epsilon^2}{16L_0}) }{2\epsilon^2\left(\ln \frac{M L_1}{L_0}+1\right) }~.
\end{equation}
\end{theorem}

Theorem~\ref{thm:lower_bound_gd_fixed} shows that in the relaxed smoothness setting, GD with any constant step size will suffer from a linear term depending on $L_1 M$.
On a side note, it is a fixed version of the lower bound in \cite{zhang2019gradient}: we provide in Appendix an explanation of errors in their lower bound and our corrected proof.

Compared with GD, our algorithm only has an exponentially decaying dependence on $L_1 M$. We consider this to be substantial merit of our algorithm. Furthermore, when $\beta_2 = 0$ in which case we recover the SignSGD with Momentum algorithm, we can even show that it completely removes the effects of the unbounded gradient norms. Also notice that in such case we actually no longer need the assumption of $M_j := \sup \left\{ \left|\PartialDerivativeGeneral{\bx}{j}\right| : F(\bx) \leq F(\bx_1)\right\}$ being finite for each $j\in[d]$ anymore, and the $\VecLInftyNorm{L_1}$ term does not appear in the final bound anymore.

We also would like to point out that this bound closely resembles the one achieved by SGD with gradient clipping algorithm~\cite{zhang2020improved} except that we consider the coordinate-wise setting: take the setting of $\beta_2 = 0$ for example, we need at most $\mathcal{O}\left(\Delta\max\left\{\frac{
\VecLOneNorm{\sigma}^2\VecLOneNorm{L_0}}{\epsilon^4},
\frac{d^2\VecLOneNorm{\sigma}^2\VecLInftyNorm{L_1}^4}{\VecLOneNorm{L_0}^3},
\frac{d\VecLInftyNorm{L_1}^2}{\VecLOneNorm{L_0}}\right\}\right)$ to get a point $\bx$ with $\|\nabla F(\bx)\|_1 \le \epsilon$ with high probability.

\textbf{Remark 1} The almost surely bounded assumption~\ref{asp:noise_as_bounded} can be relaxed to sub-gaussian noise, using standard extensions of Freedman inequality~\cite[e.g.,][]{harvey2019tight}.

\textbf{Remark 2} When $\beta_2 = 0$, we can prove an average-iterate complexity bound (see Proof of Theorem~\ref{thm:generalized_signsgd} for $\beta_2 = 0$ in Appendix~\ref{ssec:proof_generalized_signsgd}); yet, we use the min form for consistency between the two cases.

\textbf{Remark 3}
Our bound is incomparable with the one in \cite[Theorem 3.2]{zhang2020improved}.
Yet, as we said, if $L_{0,j}=L_0$ and $L_{1,j}=L_1$ for all $j\in[d]$, then the function satisfies~\eqref{eq:l0l1_global_gradient}. In this case, assuming the noise vector and the gradient vector to be dense to be able to compare the $\ell_1$-norm and the $\ell_2$-norm, we recover the same bound of~\cite[Theorem 3.2]{zhang2020improved} in terms of dependencies on $L_1$, $L_0$, and $T$. Instead, in the more general case when $L_{0,j}$ and $L_{1,j}$ are not uniform vectors, our bound allows a finer control of the unbounded smoothness.

The proof of the theorem is highly technical and it uses recent advancements in the analysis of momentum methods~\cite{cutkosky2020momentum}, key techniques to deal with the $(L_0,L_1)$ assumption~\cite{zhang2020improved}, as well as a novel and essential inductive argument to control the norm of past gradients.
We want to stress that the difficulty mainly comes from analyzing Adam-type updates when $\beta_2 > 0$, while for the other case of $\beta_2 = 0$ the proof is significantly simpler.
The full proof is in the Appendix, but here we present a proof sketch that underlines the main steps. First, we list some key lemmas we used but move their proofs to the appendix due to space constraints.

\begin{addcustomcountlemma}{\ref{lem:generalized_signsgd_bounded_updates_t}}
With notations in Algorithm~\ref{alg:generalized_signsgd}, for $\tau \le \bar{\tau} =  \frac{\sqrt{1-\beta_2}}{\eta\PNormDimension\VecLInftyNorm{L_1}}$, we have $\PNorm{\bx_{t-\tau} - \bx_{t}} \le \frac{1}{\VecLInftyNorm{L_1}}$.
\end{addcustomcountlemma}

Lemma~\ref{lem:generalized_signsgd_bounded_updates_t} limits our focus to the most recent $\bar{\tau}$ steps on which Assumption~\ref{asp:l0l1_coordinate} and Lemma~\ref{lem:desent_ineq_l0l1_coordinatewise} can apply.

\begin{addcustomcountlemma}{\ref{lem:noise_seq_generalized_signsgd}}
Assume Assumption~\ref{asp:noise_as_bounded}.
With the notation of Algorithm~\ref{alg:generalized_signsgd}, let $j \in [d]$ and $\beta_1\leq 1$. Then, with probability at least $1-3\delta$, for any $t_0\in[t]$, we have
\begin{align}
    \left|\sum^{t_0}_{\tau=1}\beta_1^{t-\tau}\left(\StocGradient{\tau}{j} - \PartialDerivative{\tau}{j}\right)\right|
    \le
    3\sigma_j\max(1, \log(1/\delta))
    + \frac{3}{\sqrt{1-\beta_1^2}}\sqrt{\sigma_j^2\max(1, \log(1/\delta))}
    \triangleq
    E_j~.
\end{align}
\end{addcustomcountlemma}

Lemma~\ref{lem:noise_seq_generalized_signsgd} is the major tool we use to handle the noise we incur during drawing stochastic gradients. It is derived based on Lemma 12 in~\cite{CutkoskyM21}.

\begin{addcustomcountlemma}{\ref{lem:generalized_signsgd_update_lower_bound}}
With the notation of Algorithm~\ref{alg:generalized_signsgd} and 
under the assumptions of Theorem~\ref{thm:generalized_signsgd}, if $\left|\PartialDerivativeGeneral{\bx_{\tau}}{j}\right|\le M_j$ holds for all $\tau\le t$ and $j\in[d]$, and $D > 0$, then, with probability at least $1-3t\delta$ we have that,
\[
\text{either $\left|\PartialDerivative{t}{j}\right| < \frac{5B_j}{D}$ or $\frac{|m_{t,j}|}{\sqrt{v_{t,j}}} \ge \frac{\rho D}{5\sqrt{1-\beta_2}}$}~,
\]
where $B_j \triangleq \frac{\eta L_{0,j} }{\sqrt{1-\beta_2}(1-\beta_1)} + \beta_1^{\bar{\tau}}(M_j + \sigma_j) + (1-\beta_1)E_j$ and $D \triangleq 1 - \frac{2\eta\PNormDimension\VecLInftyNorm{L_1}}{\sqrt{1-\beta_2}(1-\beta_1)}$.
\end{addcustomcountlemma}
Lemma~\ref{lem:generalized_signsgd_update_lower_bound} is similar to Lemma A.2 in~\cite{ZouCLG21} which considered the deterministic and smooth setting; in contrast, our proof is much more challenging in that we need to tackle both the noise and the unbounded smoothness. With this lemma, we know that either the true gradient is small or that the update of our Algorithm~\ref{alg:generalized_signsgd} can be lower bounded.

\begin{addcustomcountlemma}{\ref{lem:err_unravel_signsgdmom}}
Under Assumptions~\ref{asp:objective_func}, \ref{asp:l0l1_coordinate}, and \ref{asp:noise_as_bounded}, using the hyperparameters in Theorem~\ref{thm:generalized_signsgd}, denoting $\alpha = 1 - \beta_1$ and $\boldepsilon_t = \bm_t - \nabla F(\bx_t)$, for all $t$ and $j\in[d]$ we have, with probability at least $1-3\delta$,
{{\small\begin{equation}
    |\epsilon_{t+1, j}|
    \le
    (1 - \alpha)^t\left(\alpha\sigma_j + (1-\alpha)\left|\PartialDerivative{1}{j}\right|\right) +
    \frac{\eta L_{0,j}}{\alpha}
    + \alpha E_j
    + (1 - \alpha)\eta\PNormDimension L_{1,j}\sum^{t-1}_{\tau=0}(1-\alpha)^{\tau}\left|\PartialDerivative{t-\tau}{j}\right|~.
\end{equation}}}
\end{addcustomcountlemma}

Lemma~\ref{lem:err_unravel_signsgdmom} shows how the use of momentum can help control the noise by choosing $\beta_1$ wisely. It is adapted from the proof of Theorem 2 in~\cite{CutkoskyM21} but with the added difficulty of unbounded smoothness.

\begin{proof}[Proof sketch of Theorem~\ref{thm:generalized_signsgd}]
Observing the formula of setting $\beta_1$, we can see that when $\VecLOneNorm{\sigma} \le \sqrt{\VecLOneNorm{L_0}\Delta}/\sqrt{T}$, $\beta_1 = 0$. As $\beta_2 < \beta_1$, Algorithm~\ref{alg:generalized_signsgd} reduces to SignSGD. In this case, the key component is Lemma~\ref{lem:err_unravel_signsgdmom} using which we are able to show that $\sum^T_{t=1}\left|m_{t,j} - \PartialDerivativeGeneral{\bx_t}{j}\right|$ can be controlled as $C_1\sum^{T}_{t=1}\left|\PartialDerivativeGeneral{\bx_t}{j}\right| + C_2$. The summation of true gradients over time can then be 
offsetted by choosing $\eta$ and $\beta_1$ wisely when we invoke the descent lemma~\ref{lem:desent_ineq_l0l1_coordinatewise}. The rest is standard.

Now for the other case in which $\VecLOneNorm{\sigma} > \sqrt{\VecLOneNorm{L_0}\Delta}/\sqrt{T}$, we take a different route.

First, notice that Assumption~\ref{asp:l0l1_coordinate} and the Descent Lemma~\ref{lem:desent_ineq_l0l1_coordinatewise} only hold when two points are not too far away. Thus, we need to restrict our attention to the recent updates (Lemma~\ref{lem:generalized_signsgd_bounded_updates_t}), beyond which we would have no control. This means we want the influence of those updates too long ago to not have a big effect on the current one. To make this happen, one natural idea is to use a bounded gradient assumption, then with the use of exponential averaging, their effect would be quickly reduced. Yet, assuming directly that all gradients are bounded would trivialize the \BoldLZeroLOne assumption. Thus, we pose a much weaker condition, assuming that $M_j := \sup \left\{ \left|\PartialDerivativeGeneral{\bx}{j}\right| : F(\bx) \leq F(\bx_1)\right\}$ being finite for each $j\in[d]$. Then, we prove that $M_j$ will provide an upper bound to all the true gradients the algorithm see. We prove it using induction, analyzing separately the case that either the true gradient is already very small and we have reached the proximity of a stationary point, or the objective function is monotonically non-increasing and the gradient remains bounded.

Having controlled the past gradients, we prove in Lemma~\ref{lem:generalized_signsgd_update_lower_bound} that the update of Algorithm~\ref{alg:generalized_signsgd} is either very small that we can pass or having a constant lower bound that we can use in the Descent Lemma~\ref{lem:desent_ineq_l0l1_coordinatewise}.

Also, considering that this is the stochastic setting, noise typically slows down convergence or can even cause the algorithm to diverge if the hyperparameters are not chosen wisely. To handle this, we invoke Freedman's inequality to show that the addition of adjacent stochastic noise almost cancels out each other and the absolute value of the sum remains controlled (Lemma~\ref{lem:noise_seq_generalized_signsgd}).

Yet, we still need another block to handle the difference between the true gradient and the momentum as we are updating in the direction of the momentum instead of the true gradient. Turns our that we can prove that $\text{sign}(m_{t,j}) = \text{sign}\left(\PartialDerivativeGeneral{\bx_t}{j}\right)$ when $\left|\PartialDerivativeGeneral{\bx_t}{j}\right|$ is not too small. As before, in the case $\left|\PartialDerivativeGeneral{\bx_t}{j}\right|$ is small, we have converged on that coordinate.

Combining all these blocks together, we are able to arrive at the final results.
\end{proof}

\section{Experiments}
\label{sec:experiments}

We conducted our experiments using PyTorch~\cite{paszke2019pytorch} on Nvidia V100 GPUs.

\subsection{Comparison with Other Optimizers}
\label{ssec:comparison}
\begin{figure}[t]
    \centering
    \includegraphics[width=\textwidth]{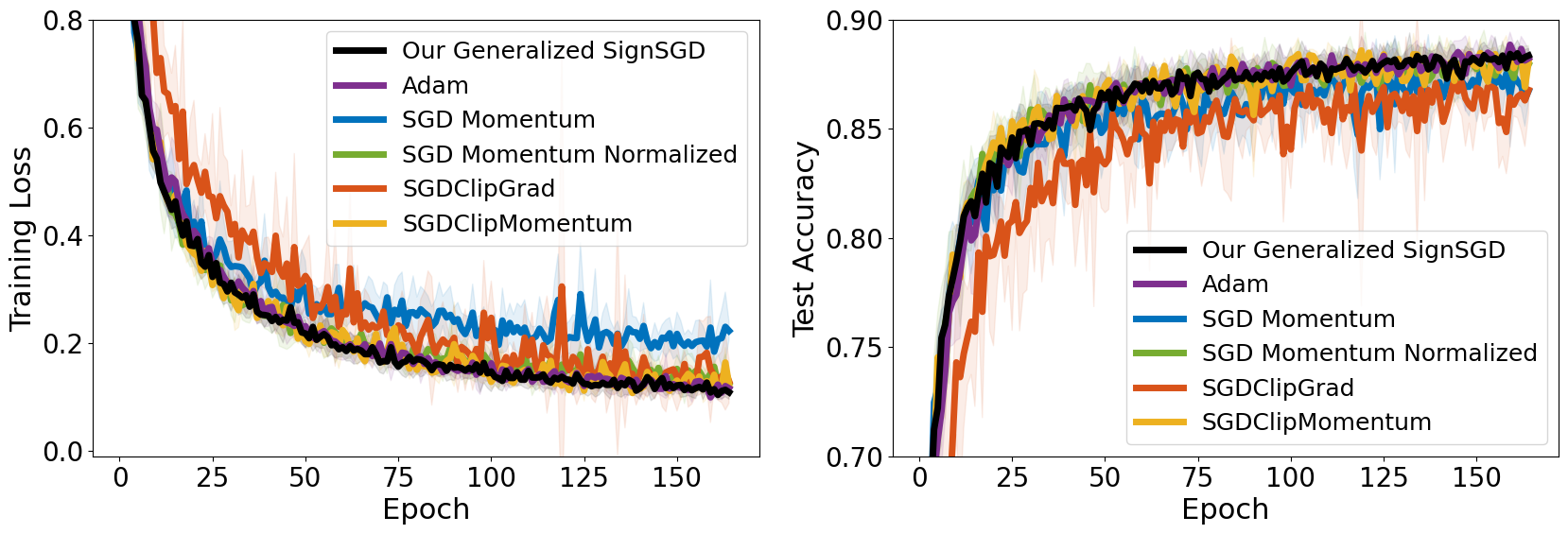}
    \caption{Training a 20-layer Resnet on CIFAR10. The shading of each curve represents the 95\% confidence interval computed across $5$ independent runs from different random seeds.}
    \label{fig:cifar10}
\end{figure}

\begin{figure}[t]
    \centering
    \includegraphics[width=\textwidth]{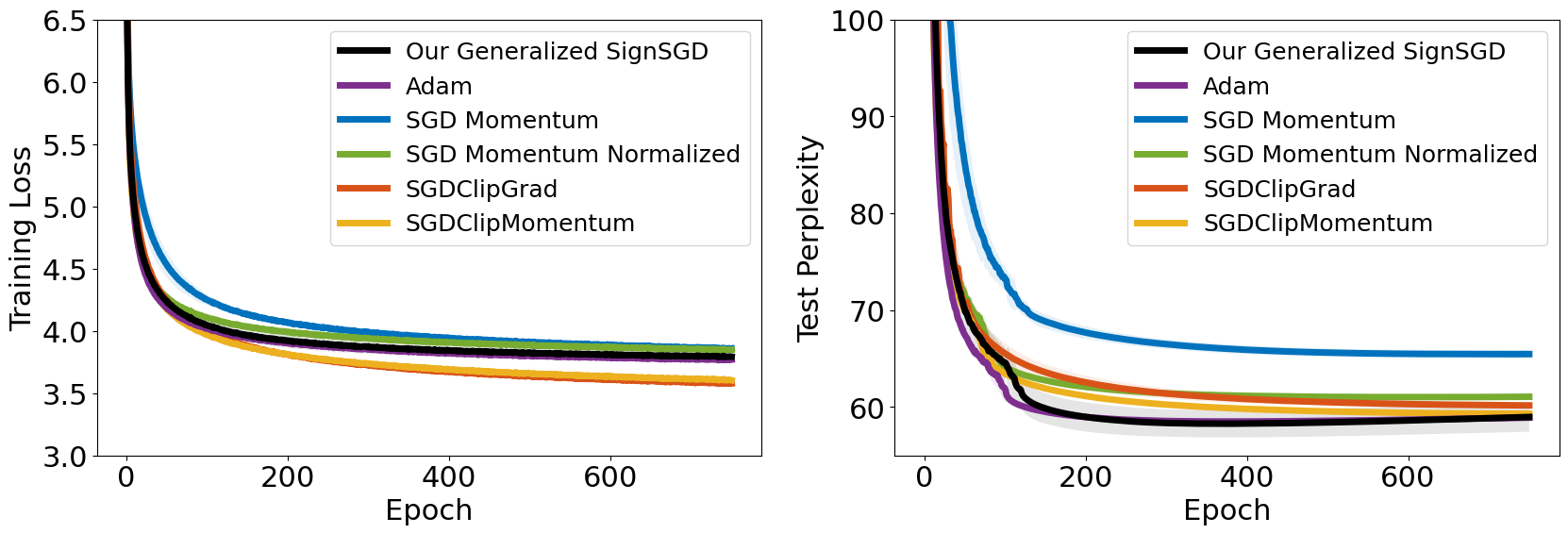}
    \caption{Training an AWD-LSTM to do language modeling (word level) on Penn Treebank. The shading of each curve represents the 95\% confidence interval over 5 independent runs.}
    \label{fig:penntreebank}
\end{figure}

To validate the efficacy of our Algorithm~\ref{alg:generalized_signsgd}, we compare it with Adam~\cite{KingmaB14}, SGD~\cite{robbins1951stochastic}, SGD Momentum Normalized~\cite{Robustness21Jin}, SGDClipGrad, and SGDClipMomentum. The latter two are from Algorithm 1 in~\cite{zhang2020improved} where SGDClipGrad corresponds to the case when $\nu = 0$ and SGDClipMomentum corresponds to when $\nu = 1$.

\textbf{Training} Unless otherwise specified, we use grid-search to fine-tune the initial learning rate for all optimizers, as well as the clipping threshold for SGDClipGrad and SGDClipMomentum, and $\beta_2$ for Adam and our algorithm, to select the one giving the best validation performance on a separated validation set. We then employ the best performing hyperparameters to train the model over all training data and report the testing performance. The testing is repeated with random seeds 5 times to eliminate the influence of stochasticity. For more details, please refer to Section~\ref{ssec:exp_details}.

\textbf{Resnet for Image Classification on CIFAR-10}
We employ the 20-layer Residual Network model~\cite{he2016deep} to do image classification on the CIFAR-10 dataset. Images are normalized per channel using the means and standard deviations
computed from all training images. We adopt the data augmentation technique following~\cite{lee2015deeply} (for training only): 4 pixels are padded on each side of an image and a 32 × 32 crop is randomly sampled from the padded image or its horizontal flip. The mini-batch size is $128$ and we train all algorithms for $164$ epochs. We do not employ any learning rate decay schedule in order to focus on the comparison of the optimizers themselves. We fixed the weight decay value to be $0.0001$ and the momentum parameter ($\beta_1)$ to be $0.9$.
Figure~\ref{fig:cifar10} and Table~\ref{tab:results} report the training and testing performance for each algorithm, showing that ours is among the best.

\textbf{LSTM for Language Modeling on Penn Treebank}
We adopt a 3-layer AWD-LSTM \cite{merity2018regularizing} to do language modeling on the Penn Treebank (PTB) dataset \cite{marcus1993building}(word level). The mini-batch size is $40$ and we trained each algorithm for $750$ epochs. Apart from the hyperparameters we stated above, we further fine-tuned the weight decay value for all algorithms noticing its significant influence on the performance. We choose the set of hyperparameters that give the smallest final validation perplexity.
We report the results in Figure~\ref{fig:penntreebank} and Table~\ref{tab:results}. It can be seen that we can match the performance of Adam while beating the others.

\begin{table*}[t]
\caption{Average final training loss and test accuracy achieved by each method when optimizing respective models on each dataset. The $\pm$ shows $95\%$ confidence intervals of the mean loss/accuracy/perplexity value over 5 runs starting from different random seeds.}
\label{tab:results}
{\small
\begin{tabular}{|c|c|c|c|c|}
\hline
\multirow{2}{*}{Methods} & \multicolumn{2}{c|}{CIFAR10} & \multicolumn{2}{c|}{Penn Treebank}\\
\cline{2-5}
& Training loss & Test accuracy & Training loss & Test perplexity \\
\hline
SGD Momentum & $0.2226 \pm 0.0169$ & $0.8674 \pm 0.0048$ & $3.8587 \pm 0.0058$ & $65.4622 \pm 0.3842$\\
\hline
SGD Momentum Normalized & $0.1262 \pm 0.0170$ & $0.8795 \pm 0.0086$ & $3.8487 \pm 0.0073$ & $61.0558 \pm 0.3224$\\
\hline
SGDClipGrad & $0.1288 \pm 0.0403$ & $0.8677 \pm 0.0106$ & $\mathbf{3.5774 \pm 0.0081}$ & $60.1604 \pm 0.2797$\\
\hline
SGDClipMomentum & $0.1220 \pm 0.0162$ & $0.8809 \pm 0.0022$ & $3.6038 \pm 0.0102$ & $59.3052 \pm 0.2798$\\
\hline
Adam & $0.1161 \pm 0.0111$ & $0.8823 \pm 0.0041$ & $3.7692 \pm 0.0062$ & $\mathbf{58.9005 \pm 0.3058}$\\
\hline
Our Algorithm~\ref{alg:generalized_signsgd} & $\mathbf{0.1086 \pm 0.0129}$ & $\mathbf{0.8835 \pm 0.0032}$ & $3.7928 \pm 0.0425$ & $58.9661 \pm 1.5218$\\
\hline
\end{tabular}}
\end{table*}

\subsection{Transformers Observe $(L_0, L_1)$-smoothness}
\label{ssec:transformer_l0l1}
For Figure~\ref{fig:global_l0l1} which verifies the original form~\eqref{eq:global_l0l1} of the $(L_0, L_1)$ condition using the norm, we followed the method in Section H.3 of~\cite{zhang2019gradient}. Specifically, given $\bx_t$ and $\bx_{t+1}$, denote $\bd := \bx_{t+1} - \bx_t$. We estimate the smoothness at $\bx_t$ by
\begin{equation}
    \hat{L}_t = \max_{\gamma\in\{\delta_1, \delta_2,\ldots,\delta_N\}}\frac{\|\nabla F(\bx_{t}+\gamma\bd) - \nabla F(\bx_t)\|_2}{\|\gamma\bd\|_2}~,\label{eq:global_smooth_estimation}
\end{equation}
where $\{\delta_1, \delta_2,\ldots,\delta_N\}$ denotes the sample locations and we use $\{\frac16, \frac26, \frac36, \frac46, \frac56\}$.

For Figure~\ref{fig:transformer_l0l1_coordinate_wise} verifying the coordinate-wise version~\eqref{eq:coordinate_wise_l0l1} of the \BoldLZeroLOne condition, note that the equation is symmetric in that if we just swap $\bx$ and $\by$ it shall still holds. Thus, during plotting, we compare $\nicefrac{\left|\PartialDerivativeGeneral{\bx_{t+1}}{j} - \PartialDerivativeGeneral{\bx_t}{j}\right|}{|x_{t+1,j} - x_{t,j}|}$ vs.~$\min\left(\left|\PartialDerivativeGeneral{\bx_t}{j}\right|, \left|\PartialDerivativeGeneral{\bx_{t+1}}{j}\right|\right)$.

Figure~\ref{fig:global_l0l1}(a) is on training a $2$-layer Transformer Encoder to do language modeling on the Wikitext-2 dataset. The implementation, settings, and parameter choices follow this.\footnote{\url{https://pytorch.org/tutorials/beginner/transformer_tutorial.html}} We only plot the first $5$ training epochs.
Figure~\ref{fig:global_l0l1}(b) and~\ref{fig:transformer_l0l1_coordinate_wise} are on training a $6$-layer Transformer~\cite{VaswaniSPUJGKP17} to do machine translation on the WMT'16 Multimodal Machine Translation Task German-English dataset.
The implementation of the transformer is forked from here\footnote{\url{https://github.com/jadore801120/attention-is-all-you-need-pytorch}} and we also follow their default settings. The mini-batch size is $256$ and we trained for $400$ epochs using Adam and report the whole training trajectory.

\subsection{Clipping does not Affect Adam's Performance}
\label{ssec:adam_clipping}
We compare clipping and non-clipping for Adam optimizer on the Wikitext-103 (103 million tokens, 180MB) \cite{wikitext103} language modeling task, with a 16-layer GPT-2 transformer model \cite{gpt2}.  This GPT-2 model has an input length of 256 tokens, 410-dimension word embedding, 16 Attention layers with 10 Attention heads and 2100 hidden dimensions. Model size is 201.58 MB. The vocabulary size is 28996. We use the hyper-parameter settings prescribed in \cite{gpt2_naacl}: batch size 256, warm up learning rate from 0 to $2.5 \times 10^{-4}$  in the first 64000 samples (i.e., 250 iterations) and then cosine-anneal learning rate to zero, on top of an Adam optimizer. It takes about 40 hours to train 200 epochs on 8 V100 GPUs. We use clipping threshold max\_norm 0.25 for the entire model as prescribed in the literature \cite{gpt2_naacl}. We also count that with this clipping scheme, clipping occurs in every single batch. As we can see from Figure~\ref{fig:gpt2_adam_clipping}, neither training loss (2.79 vs 2.76) nor perplexity score (27.92 vs 27.97) differs much in the clipping and non-clipping case, which is consistent with our theory that Adam naturally achieves gradients clipping effect.

\section{Conclusion and Limitations}
\label{sec:conclusion}
Smoothness has been a widely adopted condition for proving convergence rates of algorithms in the non-convex optimization scenario. Yet, it has been found that this assumption does not capture losses when employing some deep learning models including RNNs and LSTMs. In light of this, a relaxed smoothness assumption was proposed that aligns well with the practice. We observed that the loss surface of training using Transformers also exhibits this relaxed smoothness. Under this assumption, SGD with clipped gradient has been proven to work well. However, we found that clipping is not necessary for achieving convergence in such a setting. Indeed, we showed that a generalized SignSGD algorithm does not require explicit clipping but can almost guarantee the same bound as SGD with clipping. In the analyses, we identified the key effect of using momentum in analyzing Adam-type algorithms, that it reduces both the noise and the unbounded gradient norms. Finally, we conducted a variety of deep learning tasks showing that our algorithm can match Adam's performance while exceeding others.

\noindent\textbf{Limitations} The current work is in no way a perfect one and there are many directions worth exploring beyond it. First of all, though our algorithm could be seen as a close resemblance to the original Adam algorithm, they are still not equal. Considering the huge popularity of Adam and its established effectivity in practice, it is worth studying whether Adam in its original form can converge in the relaxed smooth setting. Second, while our Theorem~\ref{thm:generalized_signsgd} are upper bounds and cannot be directly compared between the two cases of $\beta_2$, it does suggests that $\beta_2 = 0$ minimizes the worst-case convergence rate. However, it still does not fully explain the phenomenon that a choice of $\beta_2$ close to $1$ yields better performance in using our Algorithm~\ref{alg:generalized_signsgd} as well as Adam in practice. Third, despite there are lower bounds showing that, for example, GD with a constant step size can be arbitrarily worse than GD with clipping, it would be more meaningful to study whether the relaxed smooth condition is inherently more difficult, possibly by establishing a lower bound for all first-order optimization algorithms. Fourth, we did show that Transformers observe the relaxed smoothness condition, but we consider it more beneficial to research in-depth what properties or structures make a model satisfy such conditions. Finally, when conducting our experiments, we observed that the weight decay value plays a prominent role in each optimizer's performance, and that the best weight decay value varies for different optimizers. Thus, one potential direction would be to explore different ways of incorporating the regularization in a way to preserve the scale-freeness~\cite{OrabonaP15,OrabonaP18} of Algorithm~\ref{alg:generalized_signsgd}, just as AdamW~\cite{loshchilov2018decoupled} does~\cite{zhuang2022understanding}.

\section*{Acknowledgements}
Michael Crawshaw is supported by the Institute for Digital Innovation fellowship from George Mason University. Michael Crawshaw and Mingrui Liu are both supported by a grant from George Mason University. Francesco Orabona is supported by the National Science Foundation under the grants no. 1908111 ``AF: Small: Collaborative Research: New Representations for Learning Algorithms and Secure Computation'', no. 2022446 ``Foundations of Data Science Institute'', and no. 2046096 “CAREER: Parameter-free Optimization Algorithms for Machine Learning''.

\bibliography{main}
\bibliographystyle{plainnat}

\clearpage
\appendix
\section{Appendix}
\label{sec:appendices}

\subsection{Properties of the coordinate-wise \BoldLZeroLOne Assumption~\ref{asp:l0l1_coordinate}}
\label{ssec:property_l0l1}

Below, we prove the descent lemma for coordinate-wisely \BoldLZeroLOne-smooth functions satisfying Assumption~\ref{asp:l0l1_coordinate}.

\begin{addcustomcountlemma}{\ref{lem:desent_ineq_l0l1_coordinatewise}}
Let $F$ be \BoldLZeroLOne-smooth coordinate-wisely. For any $\bx, \by\in\R^d$ for which $\PNorm{\bx - \by} \le \frac{1}{\VecLInftyNorm{L_1}}$, we have
\begin{equation}
    F(\by)
    \leq
    F(\bx) + \left\langle\nabla F(\bx), \by-\bx\right\rangle + \sum^d_{j=1}\frac{ \left(\frac{L_{0,j}}{\PNormDimension}+ L_{1,j}\left|\frac{\partial F}{\partial x_j}(\bx)\right|\right)\PNorm{\by - \bx}}{2}|y_j-x_j|~.
\end{equation}
\end{addcustomcountlemma}
\begin{proof}[Proof of Lemma~\ref{lem:desent_ineq_l0l1_coordinatewise}]
\begin{align}
    F(\by)
    &=
    F(\bx) + \int^1_0 \langle \nabla F(\bx + u(\by - \bx)), \by - \bx\rangle du\\
    &=
    F(\bx) + \left\langle\nabla F(\bx), \by-\bx\right\rangle + \int^1_0 \langle \nabla F(\bx + u(\by - \bx)) - \nabla F(\bx), \by - \bx \rangle du\\
    &\le
    F(\bx) + \left\langle\nabla F(\bx), \by-\bx\right\rangle + \left|\int^1_0 \langle \nabla F(\bx + u(\by - \bx)) - \nabla F(\bx), \by - \bx) \rangle du\right|\\
    &\le
    F(\bx) + \left\langle\nabla F(\bx), \by-\bx\right\rangle + \int^1_0 \left|\langle\nabla F(\bx + u(\by - \bx)) - \nabla F(\bx), \by - \bx \rangle \right|du\\
    &\le
    F(\bx) + \left\langle\nabla F(\bx), \by-\bx\right\rangle + \int^1_0\sum^d_{j=1} \left|\left[\frac{\partial F}{\partial x_j}(\bx + u(\by - \bx)) - \frac{\partial F}{\partial x_j}(\bx)\right] (y_j - x_j)\right|du\\
    &\le
    F(\bx) + \left\langle\nabla F(\bx), \by-\bx\right\rangle + \int^1_0\sum^d_{j=1} \left|\frac{\partial F}{\partial x_j}(\bx + u(\by - \bx)) - \frac{\partial F}{\partial x_j}(\bx)\right| |y_j - x_j|du\\
    &=
    F(\bx) + \left\langle\nabla F(\bx), \by-\bx\right\rangle + \sum^d_{j=1}\int^1_0 \left|\frac{\partial F}{\partial x_j}(\bx + u(\by - \bx)) - \frac{\partial F}{\partial x_j}(\bx)\right| |y_j - x_j|du\\
    &\le
    F(\bx) + \left\langle\nabla F(\bx), \by-\bx\right\rangle + \sum^d_{j=1}\int^1_0 u \left(\frac{L_{0,j}}{\PNormDimension} + L_{1,j}\left|\frac{\partial F}{\partial x_j}(\bx)\right|\right)\PNorm{\by - \bx}|y_j - x_j| du\\
    & =
    F(\bx) + \left\langle\nabla F(\bx), \by-\bx\right\rangle + \sum^d_{j=1}\frac{\left(\frac{L_{0,j}}{\PNormDimension}+ L_{1,j}\left|\frac{\partial F}{\partial x_j}(\bx)\right|\right)\PNorm{\by - \bx}}{2}|y_j-x_j|~,
\end{align}
where the second inequality uses the fact that $\left|\int^b_a F(x)dx\right|\le\int^b_a|F(x)|dx$ and the final one is due to Assumption~\ref{asp:l0l1_coordinate}.
\end{proof}

The following Lemma shows that our coordinate-wise \BoldLZeroLOne smooth assumption~\ref{asp:l0l1_coordinate} is equivalent to the original $(L_0, L_1)$ smooth assumption~\eqref{eq:global_l0l1} at least in $1$-d case.
\begin{lemma}
\label{lem:equivalence_l0l1_asps}
Let $F: \R\rightarrow\R$ be a twice continuously differentiable function. Then if (1) there exists some $K_0, K_1 \ge 0$ such that it holds for any $x, y\in\R$ with $|y - x| \le \frac{1}{K_1}$ that $|F^{\prime}(y) - F^{\prime}(x)| \le (K_0 + K_1|F^{\prime}(x)|)|y - x|$, then (2) there exists some $L_0, L_1 \ge 0$ such that it holds for any $x\in\R^d$ that $|F^{\prime\prime}(x)| \le L_0+ L_1|F^{\prime}(x)|$, and vice versa.
\end{lemma}
\begin{proof}[Proof of Lemma~\ref{lem:equivalence_l0l1_asps}]
\textbf{(1) $\Rightarrow$ (2)} By definition, for any $x\in\R$, we know that
\begin{align}
    F^{\prime\prime}(x)
    &=
    \underset{h\rightarrow 0}{\text{lim}}\frac{F^{\prime}(x + h) - F^{\prime}(x)}{h}
    \le
    \underset{h\rightarrow 0}{\text{lim}}\frac{|F^{\prime}(x + h) - F^{\prime}(x)|}{|h|}\\
    &\le
    \underset{h\rightarrow 0}{\text{lim}}\frac{(K_0 + K_1|F^{\prime}(\bx)|)|h|}{|h|}
    =
    K_0 + K_1|F^{\prime}(x)|~.
\end{align}

\textbf{(2) $\Rightarrow$ (1)} This is a special case for $1$-d and $c=1$ of Corollary A.4 in~\cite{zhang2020improved}.
\end{proof}

\begin{lemma}
\label{lem:l0l1_recover}
When $L_{0,j} = L_0$ and $L_{1,j} = L_1$ for all $j\in[d]$, Assumption~\ref{asp:l0l1_coordinate} implies~\eqref{eq:l0l1_global_gradient} (up to constants).
\end{lemma}

\begin{proof}
Suppose all $L_{0,j}, L_{1,j}$ are the same across $j$, then we have
\begin{align}
    \|\nabla F(\by) - \nabla F(\bx)\|_2
    &=
    \sqrt{\sum^d_{j=1}
    \left|\frac{\partial F}{\partial x_j}(\by) - \frac{\partial F}{\partial x_j}(\bx) \right|^2}\\
    &\leq
    \sqrt{\sum^d_{j=1}\left(\frac{L_{0,j}}{\sqrt{d}} + L_{1,j}\left|\frac{\partial F}{\partial x_j}(\bx)\right|\right)^2\times\|\by - \bx\|_2^2}\\
    &\leq
    \sqrt{\sum^d_{j=1}\left(\frac{2L_{0,j}^2 }{d}+ 2L_{1,j}^2\left|\frac{\partial F}{\partial x_j}(\bx)\right|^2\right)\times\|\by - \bx\|_2^2}\\
    &\leq
    \sqrt{2}L_0\|\by - \bx\|_2 + \sqrt{2}L_1\|\by - \bx\|_2\sqrt{\sum^d_{j=1}\left|\frac{\partial F}{\partial x_j}(\bx)\right|^2}\\
    &=
    \left(\sqrt{2}L_0 + \sqrt{2}L_1\|\nabla F(\bx)\|_2\right)\times \|\by - \bx\|_2~.\qedhere
\end{align}
\end{proof}

\subsection{Proof of Lower Bound}
\label{ssec:proof_lower_bound_gmd}

\begin{addcustomcounttheorem}{\ref{thm:lower_bound_gd_fixed}}
Fix $\epsilon > 0, L_0>0, L_1>0, M\geq \max(\frac{L_0}{L_1},\epsilon)$, and $x_0 \in \mathbb{R}$. Pick any constant learning rate $\eta$ for GD, with the knowledge of the above constants. Then, there exists a 1-d $(L_0,
L_1)$-smooth function, bounded from below by $f^*$ (finite), and such that $ \sup \{ |
f'(x)| : f(x) \leq f(x_0)\} \leq M$ on which the number of iterations $T$ of GD with learning rate $\eta$ to guarantee $| f'(x_T)|< \epsilon$ is at least
\begin{equation}
    \frac{ M L_1 (f(x_0) - f^*-\frac{15 \epsilon^2}{16L_0}) }{2\epsilon^2\left(\ln \frac{M L_1}{L_0}+1\right) }~.
\end{equation}
\end{addcustomcounttheorem}
\begin{proof}[Proof of Theorem~\ref{thm:lower_bound_gd_fixed}]
By Lemma~\ref{lem:equivalence_l0l1_asps}, we know that, in 1-d case, our coordinate-wise \BoldLZeroLOne assumption~\ref{asp:l0l1_coordinate} is equivalent as the original one~\eqref{eq:global_l0l1}. Thus, without loss of generality, we use the original condition~\eqref{eq:global_l0l1} in the proof. We will construct two different $(L_0,L_1)$-smooth functions based on the value of $\eta$.

\textbf{Case $\eta > \frac{2}{M L_1}\left(\ln \frac{M L_1}{L_0}+1\right)$.}
In this case, we can construct a function on which GD does not converge, hence the lower bound is trivially true.
Consider the function
\begin{equation*}
    f(x) = \begin{cases}
        L_0\frac{e^{-L_1 x-1}}{L_1^2} & x < -\frac{1}{L_1} \\
        \\
        L_0 \frac{x^2}{2} + \frac{L_0}{2 L_1^2} & x \in [-\frac{1}{L_1}, \frac{1}{L_1}] \\
        \\
        L_0\frac{e^{L_1 x-1}}{L_1^2} & x > \frac{1}{L_1}
    \end{cases}
\end{equation*}
Note that $f$ is $(L_0,L_1)$-smooth.
Without loss of generality, we can assume $x_0=\frac{1}{L_1}\left(\ln \frac{M L_1}{L_0}+1\right)$, in fact if this is not the case we can translate the function $f$ accordingly. This setting of $x_0$ guarantees that the bound on the gradient is correct.
Moreover, with this choice, we claim the function will diverge. 
To see this, we use mathematical induction to show that $|x_{t+1}| > |x_{t}|$ and $\text{sign}(x_{t+1})\ne \text{sign}(x_{t})$ for any $t\ge0$. First, for the case when $t=0$, we have
\begin{equation}
     x_{1}
     = x_{0} - \eta f^{\prime}(x_0)
     = x_0 - \frac{\eta L_0}{L_1}e^{L_1x_0 - 1}
     = x_0 - \eta M
     < x_0 - 2x_0
     = -x_0~.
\end{equation}
Then suppose the condition holds up until $t$ and we prove for $t+1$. From the formula of $f$, we have that $\text{sign}(f^{\prime}(x)) = \text{sign}(x)$ and that $f$ is monotonically increasing with $|x|$. Thus, from the update of gradient descent which moves along the negative direction of the gradient, if we can show that $|x_{t+1}| > |x_{t}|$, then $\text{sign}(x_{t+1})\ne \text{sign}(x_{t})$. This yields
\[
|x_{t+1}| = |x_{t} - \eta f^{\prime}(x_t)| > |x_t|
\Leftarrow
\eta |f^{\prime}(x_t)| > 2 |x_t|
\Leftarrow
\eta L_0 > \frac{2|x_t| L_1}{\exp(L_1 |x_t|-1)}~.
\]
Now, note that $\psi(x)=\frac{2|x| L_1}{\exp(L_1 |x|-1)}$ is decreasing for $x>\frac{1}{L_1}$ and increasing for $x<-\frac{1}{L_1}$. Hence, we have that
\[
\eta L_0 
> \frac{2|x_0| L_1}{\exp(L_1 |x_0|-1)}
> \frac{2|x_t| L_1}{\exp(L_1 |x_t|-1)},
\]
where the first inequality is true by the choice of $x_0>\frac{1}{L_1}$ and the condition on $\eta$ and the second one is true by the induction hypothesis.

\textbf{Case $\eta\leq \frac{2}{M L_1}\left(\ln \frac{M L_1}{L_0}+1\right)$.}

Now, consider
\begin{equation*}
    f(x) = \begin{cases}
        -\epsilon x , & x < -\frac{3\epsilon}{2L_0} \\
        \frac{L_0}{2} x^2-\frac{L_0^3 x^4}{27 \epsilon^2} +\frac{9\epsilon^2}{16L_0}, & x \in [-\frac{3\epsilon}{2L_0}, \frac{3\epsilon}{2L_0}] \\
        \epsilon x, & x > \frac{3\epsilon}{2L_0}
    \end{cases}
\end{equation*}
We have that $f$ is $(L_0, 0)$-smooth, hence also $(L_0, L_1)$-smooth. Note that the presence of the fourth power makes this function twice differentiable. Moreover, the maximum gradient in this case is $\epsilon\leq M$.

As before, without loss of generality, let the initial point $x_0 = \frac{3\epsilon}{2L_0}+\Delta$, where $\Delta>0$.
We have that $f(x_0)-f^* = \epsilon\left(\Delta + \frac{3\epsilon}{2L_0}\right)-\frac{9\epsilon^2}{16 L_0}$, hence $\Delta=\frac{1}{\epsilon}\left(f(x_0)-f^*\right)-\frac{15\epsilon}{16 L_0}$.
Now, while we stay on the last branch of the function, we have
\[
x_{t+1} 
= x_t - \eta \epsilon
\geq x_t - \epsilon \frac{2}{M L_1}\left(\ln \frac{M L_1}{L_0}+1\right)~.
\]
Hence, we have that, for 
\[
t
\leq \frac{M L_1 \Delta}{2 \epsilon \left(\ln \frac{M L_1}{L_0}+1\right)}
=\frac{M L_1 \left(f(x_0)-f^*-\frac{15 \epsilon^2}{16 L_0}\right)}{2\epsilon^2 \left(\ln \frac{M L_1}{L_0}+1\right)},
\]
we guarantee $|f^{\prime}(x_t)| = \epsilon$.
\end{proof}

\paragraph{Errors in the lower bound in \cite{zhang2019gradient}}
As we said in the main text, unfortunately, the lower bound theorem in \cite{zhang2019gradient} is wrong, both statement and proof. First of all, they have a logarithm of a quantity with units, $M$, which is an undefined mathematical operation. A closer look at the proof reveals that, differently from the statement of their theorem, they construct a function with $L_0=L_1$, which explains why these terms are missing in the logarithm. Moreover, it is also unclear if the second constructed function satisfies the assumptions of the theorem.
We correct all these issues by properly scaling the constructed functions so that they always satisfy the $(L_0, L_1)$ condition and all the units are coherent. This result in the correct term inside the logarithm and the right conditions on $L_0$, $L_1$, $M$, and $\epsilon$.

\subsection{Proof of Theorem~\ref{thm:generalized_signsgd}}
\label{ssec:proof_generalized_signsgd}
We first write down some notations here that we will use heavily later for easier reference:
\begin{align}
    & \bar{\tau} = \frac{\sqrt{1-\beta_2}}{\eta\PNormDimension\VecLInftyNorm{L_1}},\quad
    \alpha = 1 - \beta_1,\quad
    \rho = 1 - \beta_2^{1/2}\beta_1^{-1},\\
    & \boldepsilon_t = \bm_t - \nabla F(\bx_t),\quad
    \tilde{\boldepsilon}_t = \bg_t - \nabla F(\bx_t),\\
    & E_j = 6\sigma_j\max(1, \log(1/\delta))
    + \frac{6}{\sqrt{1-\beta_1^2}}\sqrt{\sigma_j^2\max(1, \log(1/\delta))},\\
    & B_j = \frac{\eta L_{0,j} }{\sqrt{1-\beta_2}(1-\beta_1)}
    +
    \beta_1^{\bar{\tau}}(M_j + \sigma_j)
    + (1-\beta_1)E_j,\\
    & C_j = 1 + \frac{\eta\PNormDimension L_{1,j}}{(1-\beta_1)\sqrt{1-\beta_2}},\quad
    D = 1 - \frac{2\eta\PNormDimension\VecLInftyNorm{L_1}}{\sqrt{1-\beta_2}(1-\beta_1)},\\
    & A = \frac{\rho}{10\sqrt{1-\beta_2}}~.
\end{align}

Also, we would need the following formula many times:
\begin{equation}
    \beta_1^{\bar{\tau}}
    =
    (1 - \alpha)^{\frac{1}{\alpha}\frac{\alpha\sqrt{1-\beta_2}}{\eta\PNormDimension\VecLInftyNorm{L_1}}}
    \le
    e^{-\frac{\alpha\sqrt{1-\beta_2}}{\eta\PNormDimension\VecLInftyNorm{L_1}}},\label{eq:beta1_bar_tau}
\end{equation}
where in the first inequality we used the fact that $(1 - x)^{1/x} \le \frac{1}{e}$ for $0<x<1$.

\begin{lemma}
\label{lem:moment_decomp_generalized_signsgd}
With the notations in Algorithm~\ref{alg:generalized_signsgd}, for each coordinate $j\in[d]$ we have
\begin{align}
    m_{t,j} = (1 - \beta_1)\sum^{t}_{\tau=1}\beta_1^{t-\tau} \StocGradient{\tau}{j},\qquad
    v_{t,j} =
    (1 - \beta_2)\sum^{t}_{\tau=1}\beta_2^{t-\tau} m_{\tau, j}^2,\qquad
    \frac{|m_{t,j}|}{\sqrt{v_{t,j}}} \le \frac{1}{\sqrt{1-\beta_2}}~.
\end{align}
\end{lemma}
\begin{proof}[Proof of Lemma~\ref{lem:moment_decomp_generalized_signsgd}]
For all $t\geq 1$, we have
\begin{align}
    m_{t,j}
    &=
    \beta_1 m_{t-1, j} + (1 - \beta_1) \StocGradient{t}{j}\\
    &=
    \beta_1[\beta_1 m_{t-2, j} + (1 - \beta_1) \StocGradient{t-1}{j}] + (1 - \beta_1) \StocGradient{t}{j}\\
    &=
    \ldots
    =
    (1 - \beta_1)\sum^{t}_{\tau=1}\beta_1^{t-\tau} \StocGradient{\tau}{j}~.
\end{align}
Similarly for $v_{t,j}$. Next,
\begin{align}
    \frac{|m_{t,j}|}{\sqrt{v_{t,j}}}
    &=
    \frac{|m_{t,j}|}{\sqrt{(1 - \beta_2)\sum^{t}_{\tau=1}\beta_2^{t-\tau} m_{\tau,j}^2}}
    \le
    \frac{1}{\sqrt{1-\beta_2}}~.\qedhere
\end{align}
\end{proof}

The following lemma shows when we can apply Assumption~\ref{asp:l0l1_coordinate} and Lemma~\ref{lem:desent_ineq_l0l1_coordinatewise}.
\begin{lemma}
\label{lem:generalized_signsgd_bounded_updates_t}
With notations in Algorithm~\ref{alg:generalized_signsgd}, for $\tau \le \bar{\tau} =  \frac{\sqrt{1-\beta_2}}{\eta\PNormDimension\VecLInftyNorm{L_1}}$, we have $\PNorm{\bx_{t-\tau} - \bx_{t}} \le \frac{1}{\VecLInftyNorm{L_1}}$.
\end{lemma}
\begin{proof}[Proof of Lemma~\ref{lem:generalized_signsgd_bounded_updates_t}]
Using Lemma~\ref{lem:moment_decomp_generalized_signsgd} we have
\begin{align}
    |x_{t-\tau,j} - x_{t,j}|
    &\le 
    \sum^{\tau}_{i=1}|x_{t-i,j} - x_{t-i+1,j}|
    \le
    \frac{\eta\tau}{\sqrt{1-\beta_2}}
    \le
    \frac{1}{\PNormDimension\VecLInftyNorm{L_1}}
    \Rightarrow
    \PNorm{\bx - \by} \le \frac{1}{\VecLInftyNorm{L_1}}
    ~.\qedhere
\end{align}
\end{proof}

The following two lemmas are the major tools we use to analyze the effects of noises.
\begin{lemma}[Lemma 12,~\cite{CutkoskyM21}]
\label{lem:martingale_diff}
Suppose $X_1,\ldots,X_T$ is a martingale difference sequence in a Hilbert space such that
$\|X_t\|\le R$ almost surely for some constant $R$. Further, assume $\E_{t}[\|X_t\|^2]\le\sigma_t^2$ with probability $1$ for some constants $\sigma_t$, where $\E_t[\cdot] \triangleq \E[\cdot|\xi_1, \xi_2,\ldots,\xi_{t-1}]$ denotes the expectation conditioned on all past randomnesses. Then, with probability at least $1-3\delta$, for all $k\leq T$ we have
\begin{equation}
\left\|\sum^k_{t=1}X_t\right\|
\le
3R\max(1, \log(1/\delta))
+ 3\sqrt{\sum^k_{t=1}\sigma^2_t\max(1, \log(1/\delta))}~.    
\end{equation}
\end{lemma}

\begin{lemma}
\label{lem:noise_seq_generalized_signsgd}
Assume Assumption~\ref{asp:noise_as_bounded}.
With the notation of Algorithm~\ref{alg:generalized_signsgd}, let $j \in [d]$ and $\beta_1 < 1$. Then, with probability at least $1-3\delta$, for any $t_0\in[t]$, we have
\begin{align}
    \left|\sum^{t_0}_{\tau=1}\beta_1^{t-\tau}\left(\StocGradient{\tau}{j} - \PartialDerivative{\tau}{j}\right)\right|
    \le
    3\sigma_j\max(1, \log(1/\delta))
    + \frac{3}{\sqrt{1-\beta_1^2}}\sqrt{\sigma_j^2\max(1, \log(1/\delta))}~.
\end{align}
\end{lemma}
\begin{proof}[Proof of Lemma~\ref{lem:noise_seq_generalized_signsgd}]
Recall Assumption~\ref{asp:noise_as_bounded} and notice that $\beta_1^{t-\tau} \le 1$ for all $\tau\in[1,t]$, we know that\\ $\left|\beta_1^{t-\tau}\left(\StocGradient{\tau}{j} - \PartialDerivative{\tau}{j}\right)\right| \le \sigma_j$ almost surely. It also means $\E_{\tau}\left[\left(\beta_1^{t-\tau}\left(\StocGradient{\tau}{j} - \PartialDerivative{\tau}{j}\right)\right)^2\right] \le \beta_1^{2(t-\tau)}\sigma_j^2$.
Now, note that in Algorithm~\ref{alg:generalized_signsgd} $\bg_{\tau}$ is an unbiased estimate of $\nabla F(\bx_{\tau})$ namely $\E_{\tau}\left[\beta_1^{t-\tau}\left(\StocGradient{\tau}{j} - \PartialDerivative{\tau}{j}\right)\right]=0$. Thus, $\left\{\beta_1^{t-\tau}\left(\StocGradient{\tau}{j} - \PartialDerivative{\tau}{j}\right)\right\}_{1,\ldots,t}$ is a martingale difference sequence. Then, using Lemma~\ref{lem:martingale_diff}, with probability at least $1 - 3\delta$, we have for all $t_0\in[t]$ that
\begin{align}
    \left|\sum^{t_0}_{\tau=1}\beta_1^{t-\tau}\left(\StocGradient{\tau}{j} - \PartialDerivative{\tau}{j}\right)\right|
    &\le
    3\sigma_j\max(1, \log(1/\delta))
    + 3\sqrt{\sum^{t_0}_{\tau=1}\beta_1^{2(t-\tau)}\sigma_j^2\max(1, \log(1/\delta))}\\
    &\le
    3\sigma_j\max(1, \log(1/\delta))
    + \frac{3}{\sqrt{1-\beta_1^2}}\sqrt{\sigma_j^2\max(1, \log(1/\delta))}~.\qedhere
\end{align}
\end{proof}

The following lemma upper bounds the differences between recent true gradients and the current one.
\begin{lemma}
\label{lem:true_grad_seq_generalized_signsgd}
With the notation of Algorithm~\ref{alg:generalized_signsgd} and 
under the assumptions in Theorem~\ref{thm:generalized_signsgd}, for any $j\in [d]$ and any $t_0$ with $t - t_0 \le \bar{\tau} = \frac{\sqrt{1-\beta_2}}{\eta\PNormDimension\VecLInftyNorm{L_1}}$, we have
\begin{equation}
    \sum^{t}_{\tau=t_0}\beta_1^{t-\tau} \left|\PartialDerivative{t}{j}
    - \PartialDerivative{\tau}{j}\right|
    \le
    \left(L_{0,j} + L_{1,j}\PNormDimension \left|\PartialDerivative{t}{j}\right|\right)\frac{\eta}{(1-\beta_1)^2\sqrt{1-\beta_2}}~.
\end{equation}
\end{lemma}
\begin{proof}[Proof of Lemma~\ref{lem:true_grad_seq_generalized_signsgd}]
\begin{align}
    &\sum^{t}_{\tau=t_0}\beta_1^{t-\tau} \left|\PartialDerivative{t}{j}
    - \PartialDerivative{\tau}{j}\right|\\
    \le&
    \sum^{t}_{\tau=t_0}\beta_1^{t-\tau}\left(\frac{L_{0,j}}{\PNormDimension}+ L_{1,j} \left|\PartialDerivative{t}{j}\right|\right)\PNorm{x_{t} - x_{\tau}}\\
    \le&
    \sum^{t}_{\tau=t_0}\beta_1^{t-\tau}\left(L_{0,j} + L_{1,j}\PNormDimension \left|\PartialDerivative{t}{j}\right|\right)\frac{\eta(t-\tau)}{\sqrt{1-\beta_2}}\\
    =&
    \left(L_{0,j} + L_{1,j}\PNormDimension \left|\PartialDerivative{t}{j}\right|\right)\frac{\eta}{\sqrt{1-\beta_2}}\sum^{t}_{\tau=t_0}(t-\tau)\beta_1^{t-\tau}\\
    \le&
    \left(L_{0,j} + L_{1,j}\PNormDimension \left|\PartialDerivative{t}{j}\right|\right)\frac{\eta}{(1-\beta_1)^2\sqrt{1-\beta_2}}~,
\end{align}
where the first inequality is due to Assumption~\ref{asp:l0l1_coordinate} and Lemma~\ref{lem:generalized_signsgd_bounded_updates_t}, the second inequality uses Lemma~\ref{lem:moment_decomp_generalized_signsgd}, and the final inequality uses the fact that $\sum^{N}_{k=1}k a^k\le\frac{1}{(1-a)^2}$ for any $0<a<1$.
\end{proof}

The following lemma upper bounds a past momentum with the current one.
\begin{lemma}
\label{lem:mom_diff_generalized_signsgd}
With the notation of Algorithm~\ref{alg:generalized_signsgd} and under the assumptions of Theorem~\ref{thm:generalized_signsgd}, for any $\tau \le \bar{\tau} = \frac{\sqrt{1-\beta_2}}{\eta\PNormDimension\VecLInftyNorm{L_1}}$, with probability at least $1-3\delta$, it holds that
\begin{align}
    |m_{t-\tau, j}|
    \le
    &\beta_1^{-\tau}\left(|m_{t,j}|
    +
    \left|\PartialDerivative{t}{j}\right|
    +
    \left(L_{0,j} + L_{1,j}\PNormDimension \left|\PartialDerivative{t}{j}\right|\right)\frac{\eta}{(1-\beta_1)\sqrt{1-\beta_2}}
    + (1-\beta_1)E_j\right)~.
\end{align}
\end{lemma}
\begin{proof}[Proof of Lemma~\ref{lem:mom_diff_generalized_signsgd}]
Denoting by $\tilde{\boldepsilon}_t = \bg_t - \nabla F(\bx_t)$ and using Lemma~\ref{lem:moment_decomp_generalized_signsgd}, we have
\begin{align}
    |m_{t-\tau, j} - \beta_1^{-\tau}m_{t,j}|
    &=
    \left|(1 - \beta_1)\sum^{t}_{i=1}\beta_1^{t-\tau-i} \StocGradient{i}{j}
    -
    (1 - \beta_1)\sum^{t-\tau}_{i=1}\beta_1^{t-\tau-i} \StocGradient{i}{j}\right|\\
    &=
    (1 - \beta_1)\left|\sum^{t}_{i=t-\tau+1}\beta_1^{t-\tau-i} \StocGradient{i}{j}\right|\\
    &\le
    (1 - \beta_1)\left|\sum^{t}_{i=t-\tau+1}\beta_1^{t-\tau-i} \PartialDerivative{i}{j}\right|
    +
    (1 - \beta_1)\left|\sum^{t}_{i=t-\tau+1}\beta_1^{t-\tau-i} \tilde{\epsilon}_{i, j}\right|\label{eq:momentum_diff_generalized_signsgd}~.
\end{align}
We now upper bound the first term of~\eqref{eq:momentum_diff_generalized_signsgd} using Lemma~\ref{lem:true_grad_seq_generalized_signsgd} by using the fact that $\tau \le \frac{\sqrt{1-\beta_2}}{\eta\VecLOneNorm{L_1}}$:
\begin{align}
    \left|\sum^{t}_{i=t-\tau+1}\beta_1^{t-\tau-i} \PartialDerivative{i}{j}\right|
    &\le
    \left|\sum^{t}_{i=t-\tau+1}\beta_1^{t-\tau-i} \PartialDerivative{t}{j}\right|
    +
    \sum^{t}_{i=t-\tau+1}\beta_1^{t-\tau-i} \left|\PartialDerivative{t}{j}
    - \PartialDerivative{i}{j}\right|\\
    &\le
    \left|\PartialDerivative{t}{j}\right|\frac{\beta_1^{-\tau}}{1-\beta_1}
    +
    \left(L_{0,j} + L_{1,j}\PNormDimension \left|\PartialDerivative{t}{j}\right|\right)\frac{\eta\beta_1^{-\tau}}{(1-\beta_1)^2\sqrt{1-\beta_2}}~.
\end{align}
Finally, the second term of~\eqref{eq:momentum_diff_generalized_signsgd} can be bounded using Lemma~\ref{lem:noise_seq_generalized_signsgd} by noticing that
\begin{align}
    \left|\sum^{t}_{i=t-\tau+1}\beta_1^{t-\tau-i} \tilde{\epsilon}_{i, j}\right|
    &\le
    \beta_1^{-\tau}\left(\left|\sum^{t}_{i=1}\beta_1^{t-i} \tilde{\epsilon}_{i, j}\right|
    +
    \left|\sum^{t-\tau}_{i=1}\beta_1^{t-i} \tilde{\epsilon}_{i, j}\right|\right)~.\qedhere
\end{align}
\end{proof}

The following Lemma is adapted from~\cite{ZouCLG21}. Yet, they only considered Adam under the $L$-smooth setting and when there is no noise. The existence of noise and the relaxed smoothness assumption makes the proofs substantially more challenging. With this lemma, we know that either the true gradient is small or that the update of our Algorithm~\ref{alg:generalized_signsgd} can be lower bounded.
\begin{lemma}
\label{lem:generalized_signsgd_update_lower_bound}
With the notation of Algorithm~\ref{alg:generalized_signsgd} and 
under the assumptions of Theorem~\ref{thm:generalized_signsgd}, if $\left|\PartialDerivativeGeneral{\bx_{\tau}}{j}\right|\le M_j$ holds for all $\tau\le t$ and $j\in[d]$, and $D > 0$, then, with probability at least $1-3t\delta$ we have that,
\[
\text{either $\left|\PartialDerivative{t}{j}\right| < \frac{5B_j}{D}$ or $\frac{|m_{t,j}|}{\sqrt{v_{t,j}}} \ge \frac{\rho D}{5\sqrt{1-\beta_2}}$}~.
\]
\end{lemma}
\begin{proof}
Given that $\left|\PartialDerivativeGeneral{\bx_{\tau}}{j}\right|\le M_j$ for any $\tau\le t$ and $j\in[d]$, using Lemma~\ref{lem:moment_decomp_generalized_signsgd} and Assumption~\ref{asp:noise_as_bounded}, it is immediate to show that $|m_{t,j}|\le M_j + \sigma_j$. Then, denote $\hat{\tau} = \lfloor\bar{\tau}\rfloor = \left\lfloor\frac{\sqrt{1-\beta_2}}{\eta\PNormDimension\VecLInftyNorm{L_1}}\right\rfloor$ namely the largest integer that is no greater than $\bar{\tau}$, from Lemma~\ref{lem:moment_decomp_generalized_signsgd}, we have
\begin{align}
    \frac{|m_{t,j}|}{\sqrt{v_{t,j}}}
    &=
    \frac{|m_{t,j}|}{\sqrt{(1 - \beta_2)\sum^{t-1}_{\tau=0}\beta_2^{\tau} m_{t-\tau,j}^2}}\\
    &=
    \frac{|m_{t,j}|}{\sqrt{1 - \beta_2}\sqrt{\sum^{t-1}_{\tau=\hat{\tau}+1}\beta_2^{\tau}m_{t-\tau, j}^2 + \sum^{\hat{\tau}}_{\tau=0}\beta_2^{\tau}m_{t-\tau, j}^2}}\\
    &\ge    \frac{|m_{t,j}|}{\sqrt{1 - \beta_2}\sqrt{(M_j+\sigma_j)^2\frac{\beta_2^{\hat{\tau}+1}}{1-\beta_2} + \sum^{\hat{\tau}}_{\tau=0}\beta_2^{\tau}m_{t-\tau, j}^2}}\\
    &\ge
    \frac{|m_{t,j}|}{(M_j+\sigma_j)\beta_2^{\bar{\tau}/2} + \sqrt{1 - \beta_2}\sum^{\hat{\tau}}_{\tau=0}\beta_2^{\tau/2}| m_{t-\tau,j}|}\\
    &>
    \frac{|m_{t,j}|}{(M_j+\sigma_j)\beta_1^{\bar{\tau}} + \sqrt{1 - \beta_2}\sum^{\hat{\tau}}_{\tau=0}\beta_2^{\tau/2}| m_{t-\tau,j}|},
\end{align}
where the final inequality uses the assumption that $\sqrt{\beta_2} < \beta_1$.
Using Lemma~\ref{lem:mom_diff_generalized_signsgd} and the definition of $\rho = 1-\beta_2^{1/2}\beta_1^{-1} \in (0,1]$, with probability at least $1-3t\delta$, as we need to invoke Lemma~\ref{lem:noise_seq_generalized_signsgd} for at most $t$ times, we have
\begin{align}
    \frac{\sqrt{v_{t,j}}}{\sqrt{1-\beta_2}}
    &\le
    \left(|m_{t,j}|
    +
    (1-\beta_1)E_j\right)\sum^{\hat{\tau}}_{\tau=0}\beta_2^{\tau/2}\beta_1^{-\tau}
    +
    \frac{\beta_1^{\bar{\tau}}(M_j+\sigma_j)}{\sqrt{1-\beta_2}}\\
    &\quad+\left(\left|\PartialDerivative{t}{j}\right|
    +
    \left(L_{0,j} + L_{1,j}\PNormDimension \left|\PartialDerivative{t}{j}\right|\right)\frac{\eta}{(1-\beta_1)\sqrt{1-\beta_2}}\right)\sum^{\hat{\tau}}_{\tau=0}\beta_2^{\tau/2}\beta_1^{-\tau}\\
    &\le
    \left(|m_{t,j}| + C_j\left|\PartialDerivative{t}{j}\right| + B_j\right)\frac{1}{\rho},
\end{align}
where in the last inequality we used the fact that $\frac{1}{\rho}\geq \frac{1}{\sqrt{1-\beta_2}}$.

Thus, we consider following two cases depending on the relative size of $|m_{t,j}|$ vs.~$C_j\left|\PartialDerivative{t}{j}\right|+ B_j$.

\textbf{Case 1}: $|m_{t,j}| > C_j\left|\PartialDerivative{t}{j}\right|+ B_j$, then
\begin{equation}
    \frac{|m_{t,j}|}{\sqrt{v_{t,j}}} > \frac{\rho}{2\sqrt{1-\beta_2}}~.\label{eq:generalized_signsgd_lower_bound_m_large}
\end{equation}

\textbf{Case 2}: $|m_{t,j}| \le C_j\left|\PartialDerivative{t}{j}\right|
+ B_j$, then we have
\begin{equation}
    \sqrt{v_{t,j}}
    \le
    \frac{2\sqrt{1-\beta_2}}{\rho}\left(C_j\left|\PartialDerivative{t}{j}\right| + B_j\right)~.
\end{equation}

Also, for $|m_{t,j}|$ we have from Lemma~\ref{lem:moment_decomp_generalized_signsgd} that
\begin{align}
    |m_{t,j}|
    &=
    (1 - \beta_1)\left|\sum^{t}_{\tau=1}\beta_1^{t-\tau} \StocGradient{\tau}{j}\right|
    \ge
    (1 - \beta_1)\left|\sum^{t}_{\tau=t-\hat{\tau}}\beta_1^{t-\tau} \StocGradient{\tau}{j}\right|
    -
    (1 - \beta_1)\left|\sum^{t-\hat{\tau}-1}_{\tau=1}\beta_1^{t-\tau} \StocGradient{\tau}{j}\right|\\
    &\ge
    \underbrace{(1-\beta_1)\left|\sum^{t}_{\tau=t-\hat{\tau}}\beta_1^{t-\tau} \PartialDerivative{\tau}{j}\right|}_{R_1}
    - \underbrace{(1-\beta_1)\left|\sum^{t}_{\tau=t-\hat{\tau}}\beta_1^{t-\tau}\left(\PartialDerivative{\tau}{j} - \StocGradient{\tau}{j}\right)\right|}_{R_2}
    -
    \underbrace{(1 - \beta_1)\left|\sum^{t-\hat{\tau}-1}_{\tau=1}\beta_1^{t-\tau} \StocGradient{\tau}{j}\right|}_{R_3}~.
\end{align}
The first term can be bounded below by using Lemma~\ref{lem:true_grad_seq_generalized_signsgd} and that $\hat{\tau} + 1 \ge \bar{\tau}$:
\begin{align}
    R_1
    &\ge
    (1-\beta_1)\left|\sum^{t}_{\tau=t-\hat{\tau}}\beta_1^{t-\tau} \PartialDerivative{t}{j}\right|
    -
    (1-\beta_1)\left|\sum^{t}_{\tau=t-\hat{\tau}}\beta_1^{t-\tau}\left(\PartialDerivative{\tau}{j}-  \PartialDerivative{t}{j}\right)\right|\\
    &\ge
    \left(1-\beta_1^{\bar{\tau}}-\frac{\PNormDimension L_{1,j}\eta}{(1-\beta_1)\sqrt{1-\beta_2}}\right)\left|\PartialDerivative{t}{j}\right|
    -
    \frac{L_{0,j}\eta}{(1-\beta_1)\sqrt{1-\beta_2}}~.
\end{align}
The second term can be bounded using Lemma~\ref{lem:noise_seq_generalized_signsgd}.
Thus,
\begin{align}
    |m_{t,j}|
    &\ge
    \left(1-\beta_1^{\bar{\tau}}-\frac{\PNormDimension L_{1,j}\eta}{(1-\beta_1)\sqrt{1-\beta_2}}\right)\left|\PartialDerivative{t}{j}\right|
    -
    \frac{L_{0,j}\eta}{(1-\beta_1)\sqrt{1-\beta_2}}
    - \beta_1^{\bar{\tau}}(M_j+\sigma_j)
    - (1-\beta_1)E_j\\
    &\ge
    D\left|\PartialDerivative{t}{j}\right| - B_j~,
\end{align}
where we used~\eqref{eq:beta1_bar_tau} and that $e^{-x} \le \frac{1}{x}$ for $x > 0$.

Therefore, with probability at least $1 - 3t\delta$ we have
\begin{equation}
\label{eq:lower_bound_m_over_sqrt_v}
    \frac{|m_{t,j}|}{\sqrt{v_{t,j}}}
    \ge
    \frac{\rho\left(D\left|\PartialDerivative{t}{j}\right| - B_j\right)}{2\sqrt{1-\beta_2}\left(C_j\left|\PartialDerivative{t}{j}\right| + B_j\right)}~.
\end{equation}

Given that $D>0$, depending on the relative size of $\left|\PartialDerivative{t}{j}\right|$ vs.~$B_j$, we have following two cases.

\textbf{Case 2.1}:
$\left|\PartialDerivative{t}{j}\right| < \frac{5B_j}{D}$.

\textbf{Case 2.2}: $\left|\PartialDerivative{t}{j}\right| \ge \frac{5B_j}{D}$, using the fact that the r.h.s of \eqref{eq:lower_bound_m_over_sqrt_v} is decreasing in $B_j$, we have
\begin{align}
    \frac{|m_{t,j}|}{\sqrt{v_{t,j}}}
    &\ge \frac{\frac{4D}{5}\rho\left|\PartialDerivative{t}{j}\right|}{2\sqrt{1-\beta_2}\left(C_j+\frac{D}{5}\right)\left|\PartialDerivative{t}{j}\right|}
    \ge \frac{2\rho D}{5\sqrt{1-\beta_2}(C_j + D )}
    \ge \frac{\rho D}{5\sqrt{1-\beta_2}},
\end{align}
where in the last inequality we used the fact that $C_j+D\leq 2$. Note that $D < 1$ so the above lower bound is smaller than~\eqref{eq:generalized_signsgd_lower_bound_m_large}.
\end{proof}

The following two lemmas are for the special case of $\beta_2 = 0$.
\begin{lemma}
\label{lem:signsgdm_bounded_updates}
With choices of parameters in Theorem~\ref{thm:generalized_signsgd}, when $\beta_2 = 0$, we have $\PNorm{\bx_{t+1} - \bx_{t}} = \eta\PNormDimension \le \frac{1}{\VecLInftyNorm{L_{1}}}$.
\end{lemma}
\begin{proof}[Proof of Lemma~\ref{lem:signsgdm_bounded_updates}]
Using the fact that $\alpha \le 1$ and the condition on $T$, we have
\begin{align}
\eta &= \frac{\sqrt{\Delta\alpha}}{\sqrt{\VecLOneNorm{L_0}}\sqrt{T}}
\leq
\frac{\sqrt{\Delta}}{\sqrt{\VecLOneNorm{L_0}}\sqrt{T}}
\leq
\frac{\sqrt{\Delta}}{\sqrt{\VecLOneNorm{L_0}}}\frac{\sqrt{\VecLOneNorm{L_0}}}{10\PNormDimension\sqrt{\Delta}\VecLInftyNorm{L_1}}
\le\frac{1}{\PNormDimension\VecLInftyNorm{L_{1}}}~.
\qedhere
\end{align}
\end{proof}

The following lemma is adapted from the proof of Theorem 2 in~\cite{CutkoskyM21}.
\begin{lemma}
\label{lem:err_unravel_signsgdmom}
Under Assumptions~\ref{asp:objective_func}, \ref{asp:l0l1_coordinate}, and \ref{asp:noise_as_bounded}, using the settings of the hyperparameters in Theorem~\ref{thm:generalized_signsgd}, denoting $\alpha = 1 - \beta_1$ and $\boldepsilon_t = \bm_t - \nabla F(\bx_t)$, for all $t\geq 1$ and $j\in[d]$ we have, with probability at least $1-3\delta$,
\begin{equation}
    |\epsilon_{t+1, j}|
    \le
    (1 - \alpha)^t\left(\alpha\sigma_j + (1-\alpha)\left|\PartialDerivative{1}{j}\right|\right) +
    \frac{\eta L_{0,j}}{\alpha}
    + \alpha E_j
    + (1 - \alpha)\eta\PNormDimension L_{1,j}\sum^{t-1}_{\tau=0}(1-\alpha)^{\tau}\left|\PartialDerivative{t-\tau}{j}\right|~.
\end{equation}
\end{lemma}
\begin{proof}[Proof of Lemma~\ref{lem:err_unravel_signsgdmom}]
Denote $\tilde{\boldepsilon}_t = \bg_t - \nabla F(\bx_t)$ and $S_j(\ba, \bb) = \frac{\partial F}{\partial x_j}(\ba) - \frac{\partial F}{\partial x_j}(\bb)$. Then, from Assumption~\ref{asp:l0l1_coordinate} and Lemma~\ref{lem:signsgdm_bounded_updates}, for all $t \ge 1$ and all $j\in[d]$ we have
\begin{align}
&\boldepsilon_1 = \alpha\tilde{\boldepsilon}_1-(1-\alpha)\nabla F(\bx_1),\label{eq:init_mom_grad_diff}\\
&\PNorm{\bx_{t+1} - \bx_{t}}\le\frac{1}{\VecLInftyNorm{L_{1}}} \Rightarrow |S_j(\bx_{t+1}, \bx_t)| \le \left(\frac{L_{0,j}}{\sqrt{d}} + L_{1,j}\left|\frac{\partial F}{\partial x_j}(\bx_t)\right|\right)\PNorm{\bx_{t+1} - \bx_{t}}~.\label{eq:adjacent_step_diff}
\end{align}
We can derive the following recursive formulation for any $t \ge 1$:
\begin{align*}
    m_{t+1, j}
    & =
    (1 - \alpha)m_{t, j} + \alpha \StocGradient{t+1}{j}\\
    & =
    (1 - \alpha)\PartialDerivative{t}{j} + (1 - \alpha)\epsilon_{t, j} + \alpha\PartialDerivative{t+1}{j} + \alpha\tilde{\epsilon}_{t+1,j}\\
    & =
    \PartialDerivative{t+1}{j} + (1 - \alpha)S_j(\bx_t, \bx_{t+1}) + (1 - \alpha)\epsilon_{t, j} + \alpha\tilde{\epsilon}_{t+1,j},
\end{align*}
which implies
\begin{align}
\label{eq:recur_err}
    \epsilon_{t+1, j}
    &=
    (1 - \alpha)\epsilon_{t, j} + (1 - \alpha)S_j(\bx_t, \bx_{t+1}) + \alpha\tilde{\epsilon}_{t+1,j}~.
\end{align}
Unravel~\eqref{eq:recur_err} from $1$ to $t$ gives us
\begin{equation}
    \epsilon_{t+1, j}
    =
    (1 - \alpha)^t\epsilon_{1, j} + 
    (1 - \alpha)\sum^{t-1}_{\tau=0}(1-\alpha)^{\tau} S_j(\bx_{t-\tau}, \bx_{t+1-\tau}) + \alpha\sum^{t-1}_{\tau=0}(1-\alpha)^{\tau}\tilde{\epsilon}_{t+1-\tau,j}~.
\end{equation}
Take the absolute value of both sides, to obtain
\begin{align}
    \label{eq:err_unravel}
    |\epsilon_{t+1, j}|
    \le &
    (1 - \alpha)^t|\epsilon_{1, j}| + 
    (1 - \alpha)\sum^{t-1}_{\tau=0}(1-\alpha)^{\tau} |S_j(\bx_{t-\tau}, \bx_{t+1-\tau})| + \alpha\left|\sum^{t-1}_{\tau=0}(1-\alpha)^{\tau}\tilde{\epsilon}_{t+1-\tau,j}\right|\\
    \le &
    (1 - \alpha)^t|\epsilon_{1, j}| +
    (1 - \alpha)\sum^{t-1}_{\tau=0}(1-\alpha)^{\tau} \left(\frac{L_{0,j}}{\sqrt{d}} + L_{1,j}\left|\PartialDerivative{t-\tau}{j}\right|\right)\PNorm{\bx_{t+1-\tau} - \bx_{t-\tau}}
    + \alpha E_j\\
    \le &
    (1 - \alpha)^t|\epsilon_{1, j}| +
    (1 - \alpha)\eta L_{0,j}\sum^{t-1}_{\tau=0}(1-\alpha)^{\tau}
    + (1 - \alpha)\eta\PNormDimension L_{1,j}\sum^{t-1}_{\tau=0}(1-\alpha)^{\tau}\left|\PartialDerivative{t-\tau}{j}\right| + \alpha E_j\\
    \le &
    (1 - \alpha)^t\left(\alpha\sigma_j + (1-\alpha)\left|\PartialDerivative{1}{j}\right|\right) +
    \frac{\eta L_{0,j}}{\alpha}
    + (1 - \alpha)\eta\PNormDimension L_{1,j}\sum^{t-1}_{\tau=0}(1-\alpha)^{\tau}\left|\PartialDerivative{t-\tau}{j}\right|
    + \alpha E_j~,
\end{align}
where the second inequality uses~\eqref{eq:adjacent_step_diff} and Lemma~\ref{lem:noise_seq_generalized_signsgd}, the fourth and fifth inequalities use~\eqref{eq:adjacent_step_diff}, and the final one is due to~\eqref{eq:init_mom_grad_diff}.
\end{proof}

\begin{addcustomcounttheorem}{\ref{thm:generalized_signsgd}}
Under Assumptions~\ref{asp:objective_func},~\ref{asp:l0l1_coordinate}, and~\ref{asp:noise_as_bounded}, assume $M_j := \sup \left\{ \left|\PartialDerivativeGeneral{\bx}{j}\right| : F(\bx) \leq F(\bx_1)\right\}$ is finite for each $j\in[d]$, let $\Delta$ be any upper bound on 
$F(\bx_1) - F^*$,
$\alpha = \min\left(\frac{\sqrt{\VecLOneNorm{L_0}}\sqrt{\Delta}}{\VecLOneNorm{\sigma}\sqrt{T}}, 1\right)$, $\beta_1 = 1 - \alpha$, $\frac{\sqrt{\beta_2}}{\beta_1} < 1$,
$\rho = 1-\frac{\sqrt{\beta_2}}{\beta_1}$,
$\eta = \frac{\sqrt{\Delta\alpha}}{\sqrt{\VecLOneNorm{L_0}}\sqrt{T}}$,
for $T\ge\max\left(\frac{100d\Delta\VecLInftyNorm{L_1}^2}{(1-\beta_2)\rho^2\VecLOneNorm{L_0}}, \frac{10000d^2\Delta\VecLOneNorm{\sigma}^2\VecLInftyNorm{L_1}^4}{(1-\beta_2)^2\rho^4\VecLOneNorm{L_0}^3}\right)$, Algorithm~\ref{alg:generalized_signsgd} guarantees, with probability at least $1 - \delta$, that
\begin{align}
    \min_{t\in[T]} \ \|\nabla F(\bx_t)\|_1
    =
    &\mathcal{O}\left(\frac{\sqrt{\log(dT/\delta)}\VecLOneNorm{L_0}^{1/4}\Delta^{1/4}\VecLOneNorm{\sigma}^{1/2}}{\rho\sqrt{1-\beta_2}T^{1/4}} + \frac{\log(dT/\delta)\sqrt{\VecLOneNorm{L_0}\Delta}}{\rho\sqrt{T}}\right)\\
    &+ \mathcal{O}\left(\frac{\VecLOneNorm{M} + \VecLOneNorm{\sigma}}{\rho}\exp\left(-\frac{\sqrt{1-\beta_2}\VecLOneNorm{L_0}^{3/4}}{\PNormDimension\VecLInftyNorm{L_1}\VecLOneNorm{\sigma}^{1/2}\Delta^{1/4}}T^{1/4}\right)
    +\frac{\|\nabla F(\bx_1)\|_1}{T} \right)~.
\end{align}
Furthermore, for the case when $\beta_2 = 0$, we have the following refined guarantee:
\begin{align}
\min_{t\in[T]} \ \|\nabla F(\bx_t)\|_1
=
&\mathcal{O}\left(\frac{\sqrt{\log(dT/\delta)}\VecLOneNorm{L_0}^{1/4}\Delta^{1/4}\VecLOneNorm{\sigma}^{1/2}}{T^{1/4}}
+ \frac{\log(dT/\delta)\sqrt{\VecLOneNorm{L_0}\Delta}}{\sqrt{T}}\right)\\
& + \mathcal{O}\left(\frac{\|\nabla F(\bx_1)\|_1}{\sqrt{T}}\left(\frac{1}{\sqrt{T}} + \frac{\VecLOneNorm{\sigma}}{\sqrt{\VecLOneNorm{L_0}\Delta}}\right)
+ \frac{\VecLOneNorm{\sigma}}{T}\right)~.
\end{align}
\end{addcustomcounttheorem}

\begin{proof}[Proof of Theorem~\ref{thm:generalized_signsgd} for $\beta_2 = 0$]
From Lemma~\ref{lem:signsgdm_bounded_updates} we know that $\PNorm{\bx_{t+1} - \bx_{t}} \le \frac{1}{\VecLInftyNorm{L_{1}}}$ for all $t\in[T]$. Thus, we can apply Lemma~\ref{lem:desent_ineq_l0l1_coordinatewise} to have
\begin{align}
    &F(\bx_{t+1}) - F(\bx_t)\\
    &\le
    \langle \nabla F(\bx_t), \bx_{t+1} - \bx_t \rangle
    +
    \sum^d_{j=1}\frac{\left(\frac{L_{0,j}}{\sqrt{d}} +  L_{1,j}\left|\PartialDerivative{t}{j}\right|\right)\PNorm{\bx_{t+1}-\bx_{t}}}{2}|x_{t+1,j}-x_{t,j}|\\
    &=
    \langle \nabla F(\bx_t), -\eta\text{sign}(\bm_t) \rangle
    +
    \sum^d_{j=1}\frac{ L_{0,j}+ L_{1,j}\PNormDimension\left|\PartialDerivative{t}{j}\right|}{2}\eta^2\\
    &=
    -\eta\|\nabla F(\bx_t)\|_1
    + \eta\langle \nabla F(\bx_t), \text{sign}(\nabla F(\bx_t)) - \text{sign}(\bm_t) \rangle
    +
    \sum^d_{j=1}\frac{ L_{0,j}+ L_{1,j}\PNormDimension\left|\PartialDerivative{t}{j}\right|}{2}\eta^2\\
    &=
    -\eta\|\nabla F(\bx_t)\|_1
    + 2\eta\sum^d_{j=1}\left|\PartialDerivative{t}{j}\right|\mathbb{I}\left[ \text{sign}\left(\PartialDerivative{t}{j}\right) \neq \text{sign}(m_{t,j})\right]
    +
    \sum^d_{j=1}\frac{ L_{0,j}+ L_{1,j}\PNormDimension\left|\PartialDerivative{t}{j}\right|}{2}\eta^2\label{eq:descent_signsgdm}~,
\end{align}
where $\mathbb{I}(\cdot)$ is the indicator function and the first inequality uses Lemma~\ref{lem:desent_ineq_l0l1_coordinatewise},

Now, note that
\begin{align}
    \mathbb{I}\left[\text{sign}\left(\PartialDerivative{t}{j}\right) \neq \text{sign}(m_{t,j})\right]
    & \le
    \mathbb{I}\left[\left|\PartialDerivative{t}{j} - m_{t,j}\right| \ge \left|\PartialDerivative{t}{j}\right|\right]
    \le
    \frac{\left|\PartialDerivative{t}{j} - m_{t,j}\right|}{\left|\PartialDerivative{t}{j}\right|}~.
\end{align}

Thus, denoting by $\boldepsilon_t = \bm_t - \nabla F(\bx_t)$ gives
\begin{align}
    F(\bx_{t+1}) - F(\bx_t)
    &\le
    -\eta\|\nabla F(\bx_t)\|_1
    + 2\eta\|\boldepsilon_t\|_1
    +
    \sum^d_{j=1}\frac{L_{0,j} + L_{1,j}\PNormDimension\left|\PartialDerivative{t}{j}\right|}{2}\eta^2\\
    &=
    -\eta\|\nabla F(\bx_t)\|_1
    + 2\eta\|\boldepsilon_t\|_1
    +
    \frac{\VecLOneNorm{L_0}\eta^2}{2}
    + \frac{\eta^2\PNormDimension}{2}\sum^d_{j=1}L_{1,j}\left|\PartialDerivative{t}{j}\right|~.
\end{align}
Sum both sides over $t=1, \dots,T$, to have
\begin{align}
    F^* - F(\bx_1)
    \le
    -\eta\sum^T_{t=1}\|\nabla F(\bx_t)\|_1
    + 2\eta\sum^T_{t=1}\|\boldepsilon_t\|_1
    +
    \frac{\VecLOneNorm{L_0}\eta^2 T}{2}
    + \frac{\eta^2\PNormDimension}{2}\sum^T_{t=1}\sum^d_{j=1}L_{1,j}\left|\PartialDerivative{t}{j}\right|~. \label{eq:decomp_ssgdm}
\end{align}
Use Lemma~\ref{lem:err_unravel_signsgdmom} to bound each coordinate of $\sum^T_{t=1}\|\boldepsilon_t\|_1$:
\begin{align}
    &\sum^{T-1}_{t=0}|\epsilon_{t+1, j}|\\
    &\le
    \sum^{T-1}_{t=0}\left[(1 - \alpha)^t\left(\alpha\sigma_j + (1-\alpha)\left|\PartialDerivative{1}{j}\right|\right) +
    \frac{\eta L_{0,j}}{\alpha}
    + \alpha E_j\right]\\
    &\quad
    + (1 - \alpha)\eta\PNormDimension L_{1,j}\sum^{T-1}_{t=0}\sum^{t-1}_{\tau=0}(1-\alpha)^{\tau}\left|\PartialDerivative{t-\tau}{j}\right|\\
    &=
    \sigma_j +
    \frac{1}{\alpha}\left|\PartialDerivative{1}{j}\right|
    + \frac{\eta L_{0,j} T}{\alpha}
    + \alpha E_j T
    + (1 - \alpha)\eta \PNormDimension L_{1,j}\sum^{T-1}_{t=1}\sum^{t}_{\tau^{\prime}=1}(1-\alpha)^{t-\tau^{\prime}}\left|\PartialDerivative{\tau^{\prime}}{j}\right|\\
    &=
    \sigma_j +
    \frac{1}{\alpha}\left|\PartialDerivative{1}{j}\right|
    + \frac{\eta L_{0,j} T}{\alpha}
    + \alpha E_j T
    + (1 - \alpha)\eta\PNormDimension L_{1,j}\sum^{T-1}_{\tau^{\prime}=1}\left(\sum^{T-1}_{t=\tau^{\prime}}(1-\alpha)^{t}\right)(1-\alpha)^{-\tau^{\prime}}\left|\PartialDerivative{\tau^{\prime}}{j}\right|\\
    &\le
    \sigma_j +
    \frac{1}{\alpha}\left|\PartialDerivative{1}{j}\right|
    + \frac{\eta L_{0,j} T}{\alpha}
    + \alpha E_j T
    + \frac{(1 - \alpha)\eta\PNormDimension L_{1,j}}{\alpha}\sum^{T-1}_{t=1}\left|\PartialDerivative{t}{j}\right|~.
\end{align}

The above one holds with probability at least $1-3T\delta$ as we invoked Lemma~\ref{lem:err_unravel_signsgdmom} for $T$ times which in turns means invoking Lemma~\ref{lem:noise_seq_generalized_signsgd} for $T$ times. Yet, the above inequality would need to sum from $j=1$ to $d$, meaning in total we would invoke Lemma~\ref{lem:noise_seq_generalized_signsgd} for $dT$ times. Thus, following results hold with probability at least $1-3dT\delta$.

Now, put the above inequality back into \eqref{eq:decomp_ssgdm} to have
\begin{align}
    F^*& - F(\bx_1)\\
    \le
    & -\eta\sum^T_{t=1}\|\nabla F(\bx_t)\|_1
    +
    \frac{\VecLOneNorm{L_0}\eta^2T}{2}
    + \frac{\eta^2\PNormDimension}{2}\sum^T_{t=1}\sum^d_{j=1}L_{1,j}\left|\PartialDerivative{t}{j}\right|\\
    & + 2\eta\sum^d_{j=1}\left(\sigma_j +
    \frac{1}{\alpha}\left|\PartialDerivative{1}{j}\right|
    + \frac{\eta L_{0,j} T}{\alpha}
    + \alpha E_j T
    + \frac{(1 - \alpha)\eta\PNormDimension L_{1,j}}{\alpha}\sum^{T}_{t=1}\left|\PartialDerivative{t}{j}\right|\right)\\
    =
    & -\eta\sum^T_{t=1}\|\nabla F(\bx_t)\|_1
    +
    \frac{\VecLOneNorm{L_0}\eta^2T}{2}
    + \frac{\eta^2\PNormDimension}{2}\sum^T_{t=1}\sum^d_{j=1}L_{1,j}\left|\PartialDerivative{t}{j}\right|
    + 2\eta\VecLOneNorm{\sigma}\\
    &
    + \frac{2\eta}{\alpha}\|\nabla F(\bx_1)\|_1
    + \frac{2\eta^2 \VecLOneNorm{L_0} T}{\alpha}
    + 2\eta\alpha T\sum^{d}_{j=1}E_j
    + \frac{2\eta^2\PNormDimension}{\alpha}\sum^T_{t=1}\sum^d_{j=1}L_{1,j}\left|\PartialDerivative{t}{j}\right|\\
    =
    & -\eta\sum^T_{t=1}\|\nabla F(\bx_t)\|_1
    + 2\eta\VecLOneNorm{\sigma}
    + \frac{2\eta}{\alpha}\|\nabla F(\bx_1)\|_1
    + \left(\frac12 + \frac{2}{\alpha}\right)\VecLOneNorm{L_0}\eta^2T
    + 2\eta\alpha T\sum^{d}_{j=1}E_j\\
    &+ \left(\frac12 + \frac{2}{\alpha}\right)\eta^2\PNormDimension\sum^T_{t=1}\sum^d_{j=1}L_{1,j}\left|\PartialDerivative{t}{j}\right|~.
\end{align}

Now, using the definitions of $\eta$ and $\alpha$, and the conditions on $T$, we have
\begin{align}
    \left(\frac12 + \frac{2}{\alpha}\right)\eta^2\PNormDimension L_{1,j}
    &\leq
    \frac{\eta}{2}\left(1 + \frac{4}{\alpha}\right)\frac{\PNormDimension\sqrt{\Delta\alpha}}{\sqrt{T}}\frac{\VecLInftyNorm{L_1}}{\sqrt{\VecLOneNorm{L_0}}}
    \le
    \frac{\eta}{2}\frac{5\PNormDimension\sqrt{\Delta}}{\sqrt{\alpha T}}\frac{\VecLInftyNorm{L_1}}{\sqrt{\VecLOneNorm{L_0}}}\\
    &=
    \frac{\eta}{2}\frac{5\PNormDimension\sqrt{\Delta}}{\sqrt{ T}}\frac{\VecLInftyNorm{L_1}}{\sqrt{\VecLOneNorm{L_0}}}\cdot\max\left(\frac{\sqrt{\VecLOneNorm{\sigma}}T^{1/4}}{\VecLOneNorm{L_0}^{1/4}\Delta^{1/4}}, 1\right)\\
    &=
    \frac{\eta}{2}\max\left(\frac{5\PNormDimension\sqrt{\VecLOneNorm{\sigma}}\VecLInftyNorm{L_1}\Delta^{1/4}}{\VecLOneNorm{L_0}^{3/4}T^{1/4}},\frac{5\PNormDimension\sqrt{\Delta}}{\sqrt{ T}}\frac{\VecLInftyNorm{L_1}}{\sqrt{\VecLOneNorm{L_0}}}\right)
    \le
    \frac{\eta}{2}~.
\end{align}

Thus, we have
\begin{equation}
    F^* - F(\bx_1)
    \le
    -\frac{\eta}{2}\sum^T_{t=1}\|\nabla F(\bx_t)\|_1
    +
    2\eta\VecLOneNorm{\sigma}
    + \frac{2\eta}{\alpha}\|\nabla F(\bx_1)\|_1
    + \left(\frac12 + \frac{2}{\alpha}\right)\VecLOneNorm{L_0}\eta^2T
    + 2\eta\alpha T\sum^{d}_{j=1}E_j~.
\end{equation}
Divide both sides by $T$ and rearrange terms to give
\begin{align}
    \frac{1}{T}\sum^T_{t=1}\|\nabla F(\bx_t)\|_1
    \le
    &\frac{2}{\eta T}[F(\bx_1) - F^*]
    + \frac{4}{T}\VecLOneNorm{\sigma}
    + \frac{4}{\alpha T}\|\nabla F(\bx_1)\|_1
    +
    \frac{5}{\alpha}\VecLOneNorm{L_0}\eta\\
    &+
    24\VecLOneNorm{\sigma}(\alpha\max(1, \log(1/\delta))
    + \sqrt{\alpha}\sqrt{\max(1, \log(1/\delta))})~.
\end{align}

Now, we need to consider the following two cases:
\begin{enumerate}
    \item $\VecLOneNorm{\sigma} < \frac{\sqrt{\VecLOneNorm{L_0}}\sqrt{\Delta}}{\sqrt{T}}$: then $\alpha = 1$ and $\eta = \frac{\sqrt{\Delta}}{\sqrt{\VecLOneNorm{L_0}}\sqrt{T}}$
    \begin{align}
        \frac{1}{T}\sum^T_{t=1}\|\nabla F(\bx_t)\|_1
        &\le
        \frac{2\sqrt{\VecLOneNorm{L_0}}}{\sqrt{\Delta}\sqrt{T}}[F(\bx_1) - F^*]
        + \frac{5\VecLOneNorm{L_0}\sqrt{\Delta}}{\sqrt{\VecLOneNorm{L_0}}\sqrt{T}}
        + \frac{4\sqrt{\VecLOneNorm{L_0}}\sqrt{\Delta}}{T^{3/2}}
        + \frac{4}{T} \|\nabla F(\bx_1)\|_1\\
        &\quad
        + \frac{24\sqrt{\VecLOneNorm{L_0}}\sqrt{\Delta}(\max(1, \log(1/\delta)) + \sqrt{\max(1, \log(1/\delta))})}{\sqrt{T}}\\
        &\le
        \frac{59\max(1, \log(1/\delta))\sqrt{\VecLOneNorm{L_0}\Delta}}{\sqrt{T}} + \frac{4}{T} \|\nabla F(\bx_1)\|_1\label{eq:small_noise_bound}~.
    \end{align}
    \item $\VecLOneNorm{\sigma}\ge\frac{\sqrt{\VecLOneNorm{L_0}\Delta}}{\sqrt{T}}$: then $\alpha = \frac{\sqrt{\VecLOneNorm{L_0}\Delta}}{\VecLOneNorm{\sigma}\sqrt{T}} \le 1$ and $\eta = \frac{\Delta^{3/4}}{\VecLOneNorm{L_0}^{1/4}\sqrt{\VecLOneNorm{\sigma}}T^{3/4}}$
    \begin{align}
        &\frac{1}{T}\sum^T_{t=1}\|\nabla F(\bx_t)\|_1\\
        \le& \frac{2\VecLOneNorm{L_0}^{1/4}\sqrt{\VecLOneNorm{\sigma}}}{\Delta^{3/4}T^{1/4}}[F(\bx_1) - F^*]
        + \frac{4\VecLOneNorm{\sigma}}{T}
        +
        \frac{4\VecLOneNorm{\sigma}}{\sqrt{\VecLOneNorm{L_0}\Delta T}}\|\nabla F(\bx_1)\|_1
        +
        \frac{5\VecLOneNorm{L_0}^{1/4}\sqrt{\VecLOneNorm{\sigma}}\Delta^{1/4}}{T^{1/4}}\\
        &+ \frac{24\max(1, \log(1/\delta))\sqrt{\VecLOneNorm{L_0}\Delta}}{\sqrt{T}}
        +\frac{24\sqrt{\max(1, \log(1/\delta))}\VecLOneNorm{L_0}^{1/4}\sqrt{\VecLOneNorm{\sigma}}\Delta^{1/4}}{T^{1/4}}\\
        \le&
        \frac{31\sqrt{\max(1, \log(1/\delta))}\VecLOneNorm{L_0}^{1/4}\Delta^{1/4}\sqrt{\VecLOneNorm{\sigma}}}{T^{1/4}}
        + \frac{24\max(1, \log(1/\delta))\sqrt{\VecLOneNorm{L_0}\Delta}}{\sqrt{T}}\\
        &+ \frac{4\VecLOneNorm{\sigma}\|\nabla F(\bx_1)\|_1}{\sqrt{\VecLOneNorm{L_0}\Delta T}}
        + \frac{4\VecLOneNorm{\sigma}}{T}~.
    \end{align}
\end{enumerate}

Put $\delta^{\prime} = \frac{\delta}{3dT}$ concludes the proof.
\end{proof}

\begin{proof}[Proof of Theorem~\ref{thm:generalized_signsgd} for general $\beta_2$]
Note that when $\VecLOneNorm{\sigma}\le\frac{\sqrt{\VecLOneNorm{L_0}\Delta}}{\sqrt{T}}$, $\alpha = 1$, $\beta_2 < \beta_1^2 = 0$. Then our Generalized SignSGD algorithm~\ref{alg:generalized_signsgd} reduces to the SignSGD algorithm, and thus has the same guarantee of~\eqref{eq:small_noise_bound}. Therefore, we only consider the other case when $\VecLOneNorm{\sigma}\ge\frac{\sqrt{\VecLOneNorm{L_0}\Delta}}{\sqrt{T}}$.

Note that the only randomness comes from evaluating stochastic gradients. In the following proof, we will need to invoke Lemma~\ref{lem:noise_seq_generalized_signsgd} for $T$ times for each coordinate $j\in[d]$. Therefore, the following results hold with probability at least $1 - 3dT\delta$. For simplicity, we use the term "with high probability" later in the proof to denote this.

We derive the following quantity which will be used multiple times later:
\begin{equation}
    \frac{\eta}{\alpha}
    =
    \frac{1}{\sqrt{\VecLOneNorm{L_0}}}\cdot\frac{\sqrt{\Delta}}{\sqrt{T}} \cdot \frac{1}{\sqrt{\alpha}}
    =
    \frac{1}{\sqrt{\VecLOneNorm{L_0}}} \cdot\frac{\sqrt{\Delta}}{\sqrt{T}} \cdot \frac{\VecLOneNorm{\sigma}^{1/2}T^{1/4}}{\VecLOneNorm{L_0}^{1/4}\Delta^{1/4}}
    =
    \frac{\VecLOneNorm{\sigma}^{1/2}\Delta^{1/4}}{\VecLOneNorm{L_0}^{3/4}T^{1/4}}~.\label{eq:eta_over_alpha}
\end{equation}

First, from Lemma~\ref{lem:generalized_signsgd_update_lower_bound} we have $D = 1-\frac{2\PNormDimension\VecLInftyNorm{L_{1}}\eta}{(1-\beta_1)\sqrt{1-\beta_2}}$. Then, from the choice of the hyperparameters, we have
\begin{equation}
    \frac{\PNormDimension\VecLInftyNorm{L_{1}}\eta}{(1-\beta_1)\sqrt{1-\beta_2}}
    =
    \frac{\VecLInftyNorm{L_{1}}}{\sqrt{1-\beta_2}} \cdot \frac{\PNormDimension\VecLOneNorm{\sigma}^{1/2}\Delta^{1/4}}{\VecLOneNorm{L_0}^{3/4}T^{1/4}}
    \le
    \frac{\rho}{10}
    \le
    \frac{1}{10}~.\label{eq:L1_eta_over_alpha}
\end{equation}
Thus, we have $D \ge 1 - \frac{1}{5} \ge \frac12$ and, as $\sqrt{\beta_2}/\beta_1 < 1$,
\begin{equation}
    \frac{\rho D}{5\sqrt{1-\beta_2}}
    \ge
    \frac{\rho}{10\sqrt{1-\beta_2}}
    =
    A~.\label{eq:update_lower_bound_generalized_signsgd}
\end{equation}

Also, for those coordinates with small gradients $\left|\PartialDerivative{t}{j}\right| < \frac{5B_j}{D} \le 10B_j$, we have
\begin{align}
    \PartialDerivative{t}{j}\cdot(x_{t+1, j} - x_{t,j})
    =
    &-\eta\PartialDerivative{t}{j}\cdot\frac{m_{t,j}}{\sqrt{v_{t,j}}}\\
    =
    &-A\eta\left|\PartialDerivative{t}{j}\right| + \eta\left|\PartialDerivative{t}{j}\right|\cdot\left(A - \text{sign}\left(\PartialDerivative{t}{j}\right)\frac{m_{t,j}}{\sqrt{v_{t,j}}}\right)\\
    \le
    &-A\eta\left|\PartialDerivative{t}{j}\right| + 10B_j\eta\left(\frac{1}{\sqrt{1-\beta_2}} + A\right)~. \label{eq:inner_prod_upper_bound_grad_small}
\end{align}

We are now ready to prove the theorem. We will need to use Lemma~\ref{lem:generalized_signsgd_update_lower_bound}, hence we first need to show that all past true gradients are bounded by $M_j$, namely that, for any $t$, $\left|\PartialDerivativeGeneral{\bx_{\tau}}{j}\right| \le M_j$ holds for all $\tau \le t$ and all $j\in[d]$. From the definition of $M_j$ stated in the theorem, in order to guarantee this, we only need to prove that $F(\bx_{\tau}) \le F(\bx_1)$ for all $\tau\le t$. We will prove this by induction.

For $t = 1$ the condition is trivially true.

We then assume that the condition holds for $t$ and prove based on this that it still holds for $t+1$.

For those coordinates with $\left|\PartialDerivative{t}{j}\right| \ge \frac{5B_j}{D}$,
denote $\hat{\tau} = \lfloor\bar{\tau}\rfloor = \left\lfloor\frac{\sqrt{1-\beta_2}}{\eta\PNormDimension\VecLInftyNorm{L_1}}\right\rfloor$, with high probability, we have
\begin{align}
    \left|m_{t,j} - \PartialDerivative{t}{j}\right|
    &=
    \left|(1 - \beta_1)\sum^{t}_{\tau=1}\beta_1^{t-\tau} \StocGradient{\tau}{j} - \PartialDerivative{t}{j} \right|\\
    &\le
    \left|(1 - \beta_1)\sum^{t-\hat{\tau}-1}_{\tau=1}\beta_1^{t-\tau} \StocGradient{\tau}{j}\right| + \left|(1 - \beta_1)\sum^{t}_{\tau=t-\hat{\tau}}\beta_1^{t-\tau} \StocGradient{\tau}{j} - \PartialDerivative{t}{j} \right|\\
    &\le
    \beta_1^{\bar{\tau}}(M_j + \sigma_j)
    +
    \beta_1^{\bar{\tau}}\left|\PartialDerivative{t}{j}\right|\\
    &\quad+(1-\beta_1)\left|\sum^{t}_{\tau=t-\hat{\tau}}\beta_1^{t-\tau} \left(\PartialDerivative{\tau}{j}- \PartialDerivative{t}{j}\right)\right|\\
    &\quad+
    (1-\beta_1)\left|\sum^{t}_{\tau=t-\hat{\tau}}\beta_1^{t-\tau}\left(\StocGradient{\tau}{j} - \PartialDerivative{\tau}{j}\right)\right|\\
    &\le
    \left(1 - D\right)\left|\PartialDerivative{t}{j}\right| + B_j\\
    &\le
    \left(1 - \frac{4D}{5}\right)\left|\PartialDerivative{t}{j}\right|
    \le
    \left|\PartialDerivative{t}{j}\right|, \label{eq:equal_sign_moment_grad}
\end{align}
where the first equality comes from Lemma~\ref{lem:moment_decomp_generalized_signsgd}; for the second inequality, the third term can be bounded using Lemma~\ref{lem:true_grad_seq_generalized_signsgd}, and the final term can be bounded using Lemma~\ref{lem:noise_seq_generalized_signsgd}; for the third inequality, we used~\eqref{eq:beta1_bar_tau} and that $e^{-x}\le\frac{1}{x}$ for $x > 0$.
This inequality implies that $\text{sign}(m_{t,j}) = \text{sign}\left(\PartialDerivative{t}{j}\right)$ with high probability.

Denote $\mathcal{U}_t = \left\{j\in[d]: \left|\PartialDerivative{t}{j}\right| \ge \frac{5B_j}{D}\right\}$. From the choices of hyperparameters we can show that $1\le\bar{\tau}$ which means $\PNorm{\bx_{t+1} - x_{t}}\le\frac{1}{\VecLInftyNorm{L_{1}}}$ (Lemma~\ref{lem:generalized_signsgd_bounded_updates_t}). Thus, using Lemma~\ref{lem:desent_ineq_l0l1_coordinatewise}, with high probability we have
\begin{align}
    & F(\bx_{t+1}) - F(\bx_t)\\
    &\le
    \langle \nabla F(\bx_t), \bx_{t+1} - \bx_t \rangle
    +
    \sum^d_{j=1}\frac{\left(\frac{L_{0,j}}{\sqrt{d}} +  L_{1,j}\left|\PartialDerivative{t}{j}\right|\right)\PNorm{\bx_{t+1}-\bx_{t}}}{2}|x_{t+1,j}-x_{t,j}|\\
    &=
    \sum^d_{j=1}\left(-\PartialDerivative{t}{j} \cdot \eta\text{sign}(m_{t,j})\frac{|m_{t,j}|}{\sqrt{v_{t,j}}}
    +
    \frac{\left( \frac{L_{0,j}}{\sqrt{d}}+ L_{1,j}\left|\PartialDerivative{t}{j}\right|\right)\PNorm{\bx_{t+1}-\bx_{t}}}{2}|x_{t+1,j}-x_{t,j}|\right)\\
    &\le
    \sum_{j\in\mathcal{U}_t}-\eta\left|\PartialDerivative{t}{j}\right|\cdot \frac{|m_{t,j}|}{\sqrt{v_{t,j}}}
    +
    \sum_{j\not\in\mathcal{U}_t}\left(-A\eta\left|\PartialDerivative{t}{j}\right|+
    10\eta B_j\left(\frac{1}{\sqrt{1-\beta_2}} + A\right)\right)
    +
    \sum^d_{j=1}\frac{ L_{0,j}+ L_{1,j}\PNormDimension\left|\PartialDerivative{t}{j}\right|}{2(1-\beta_2)}\eta^2\\
    &\le
    -A\eta\|\nabla F(x_t)\|_1
    +
    \sum^d_{j=1}\frac{ L_{0,j}+ L_{1,j}\PNormDimension\left|\PartialDerivative{t}{j}\right|}{2(1-\beta_2)}\eta^2
    + 10\eta\left(\frac{1}{\sqrt{1-\beta_2}} + A\right)\sum^d_{j=1}B_j, \label{eq:descent_generalized_signsgd}
\end{align}
where the second inequality uses~\eqref{eq:inner_prod_upper_bound_grad_small},~\eqref{eq:equal_sign_moment_grad}, and Lemma~\ref{lem:moment_decomp_generalized_signsgd}, and the third inequality uses Lemma~\ref{lem:generalized_signsgd_update_lower_bound} and~\eqref{eq:update_lower_bound_generalized_signsgd}.

Now, noticing the conditions on $\eta$, $\alpha$, $\beta_2 < \beta_1^2 < \beta_1$, and $T$, use~\eqref{eq:eta_over_alpha} to have
\begin{align}
    \frac{\eta^2\PNormDimension L_{1,j}}{2(1-\beta_2)}
    \le
    \frac{\eta\PNormDimension\VecLInftyNorm{L_{1}}}{2}\cdot\frac{\eta}{\alpha}
    =
    \frac{\eta}{2}\frac{\PNormDimension\VecLInftyNorm{L_{1}}\VecLOneNorm{\sigma}^{1/2}\Delta^{1/4}}{\VecLOneNorm{L_0}^{3/4}T^{1/4}}
    \le
    \frac{\eta}{2}\frac{\rho\sqrt{1-\beta_2}}{10}
    \le
    \frac{\eta}{2}A~.
\end{align}
Thus,~\eqref{eq:descent_generalized_signsgd} becomes
\begin{equation}
    F(\bx_{t+1}) - F(\bx_t)
    \le
    -\frac{A\eta}{2}\|\nabla F(x_t)\|_1
    +
    \frac{\eta^2\VecLOneNorm{L_{0}}}{2(1-\beta_2)}
    + 10\eta\left(\frac{1}{\sqrt{1-\beta_2}} + A\right)\sum^d_{j=1}B_j~. \label{eq:main_descent_generalized_signsgd}
\end{equation}
Therefore, either
\begin{equation}
\|\nabla F(\bx_t)\|_1\le \frac{\eta\VecLOneNorm{L_0}}{A(1-\beta_2)} + 20\left(\frac{1}{A\sqrt{1-\beta_2}} + 1\right)\sum^d_{j=1}B_j\label{eq:generalized_signsgd_stop_condition},
\end{equation}
or $F(\bx_{t+1}) - F(\bx_t) \le 0$.

This concludes the mathematical induction up until~\eqref{eq:generalized_signsgd_stop_condition} is met for the first time which we denote as $T_0$. 
In the following, we will explain that if \eqref{eq:generalized_signsgd_stop_condition} holds then the algorithm has found an approximate stationary point.

Now, suppose $T \le T_0$, then~\eqref{eq:main_descent_generalized_signsgd} holds all the time and we sum both sides of it from $1$ to $T$ to have, with high probability,
\begin{align}
    F^* - F(\bx_1)
    \le
    &-\frac{A\eta}{2}\sum^T_{t=1}\|\nabla F(\bx_t)\|_1 +
    \frac{\VecLOneNorm{L_0}\eta^2T}{2(1-\beta_2)}
    + 10\eta T\left(\frac{1}{\sqrt{1-\beta_2}} + A\right)\sum^d_{j=1}B_j~.
\end{align}
Rearrange terms to obtain
\begin{equation}
    \min_{t\in[T]}\|\nabla F(\bx_t)\|_1
    \le
    \frac{1}{T}\sum^T_{t=1}\|\nabla F(\bx_t)\|_1
    \le
    \frac{2}{A\eta T}[F(\bx_1) - F^*] +
    \frac{\eta\VecLOneNorm{L_0}}{A(1-\beta_2)}
    + 20\left(\frac{1}{A\sqrt{1-\beta_2}} + 1\right)\sum^d_{j=1}B_j\label{eq:generalized_signsgd_min_grad_norm_bound}~.
\end{equation}

Note that RHS of~\eqref{eq:generalized_signsgd_stop_condition} is less than RHS of~\eqref{eq:generalized_signsgd_min_grad_norm_bound}. Thus, for the other case of $T > T_0$,~\eqref{eq:generalized_signsgd_min_grad_norm_bound} still holds.

Recalling that $\rho = 1-\frac{\sqrt{\beta_2}}{\beta_1}$, $A = \frac{\rho}{10\sqrt{1-\beta_2}}$, $\beta_2 < \beta_1^2 < \beta_1$, and $B_j \triangleq
\frac{\eta L_{0,j}}{(1-\beta_1)\sqrt{1-\beta_2}} + \beta_1^{\bar{\tau}}(M_j + \sigma_j)
+ 6(1-\beta_1) \sigma_j\max(1, \log(1/\delta))
+ \frac{6(1-\beta_1)}{\sqrt{1-\beta_1^2}}\sqrt{\sigma_j^2\max(1, \log(1/\delta))}$, we have
\begin{align}
    \min_{t\in[T]}\|\nabla F(\bx_t)\|_1
    \le
    &\frac{20}{\rho\eta T}[F(\bx_1) - F^*] +
    \frac{10\eta\VecLOneNorm{L_0}}{\rho\sqrt{1-\beta_1}}
    + 20\left(\frac{10}{\rho} + 1\right)\left(\frac{\VecLOneNorm{L_{0}}\eta}{(1-\beta_1)\sqrt{1-\beta_2}} + \beta_1^{\bar{\tau}}(\VecLOneNorm{M} + \VecLOneNorm{\sigma})\right)\\
    &
    +120\VecLOneNorm{\sigma}(1-\beta_1)\left(\frac{10}{\rho} + 1\right)\left(\max(1, \log(1/\delta))
    + \frac{1}{\sqrt{1-\beta_1^2}}\sqrt{\max(1, \log(1/\delta))}\right)~.
\end{align}

When $\VecLOneNorm{\sigma}\ge\frac{\sqrt{\VecLOneNorm{L_0}\Delta}}{\sqrt{T}}$, then $\alpha = \frac{\sqrt{\VecLOneNorm{L_0}\Delta}}{\VecLOneNorm{\sigma}\sqrt{T}}$ and $\eta = \frac{\Delta^{3/4}}{\VecLOneNorm{L_0}^{1/4}\sqrt{\VecLOneNorm{\sigma}}T^{3/4}}$. Hence, we obtain
\begin{align}
    &\min_{t\in[T]}\|\nabla F(\bx_t)\|_1\\
    &\le
    \frac{20\Delta^{1/4}\VecLOneNorm{L_0}^{1/4}\VecLOneNorm{\sigma}^{1/2}T^{3/4}}{\rho T}
    + \frac{10\sqrt{\VecLOneNorm{L_0}\Delta}}{\rho\sqrt{T}}
    + 20\left(\frac{10}{\rho} + 1\right)\left(\frac{\Delta^{1/4}\VecLOneNorm{\sigma}^{1/2}\VecLOneNorm{L_{0}} }{\sqrt{1-\beta_2}\VecLOneNorm{L_0}^{3/4}T^{1/4}} + \beta_1^{\bar{\tau}}(\VecLOneNorm{M} + \VecLOneNorm{\sigma})\right)\\
    &
    \quad +120\left(\frac{10}{\rho} + 1\right)\left(\frac{\sqrt{\VecLOneNorm{L_0}\Delta}}{\sqrt{T}}\max(1, \log(1/\delta))
    + \frac{\Delta^{1/4}\VecLOneNorm{L_0}^{1/4}\VecLOneNorm{\sigma}^{1/2}}{T^{1/4}}\sqrt{\max(1, \log(1/\delta))}\right)\\
    &\le
    \left(\frac{20}{\rho T^{1/4}} + \left(\frac{10}{\rho} + 1\right)\frac{20+120\sqrt{\max(1, \log(1/\delta))}}{\sqrt{1-\beta_2}T^{1/4}}\right)\Delta^{1/4}\VecLOneNorm{L_0}^{1/4}\VecLOneNorm{\sigma}^{1/2}\\
    & \quad +\left(\frac{10}{\rho}
    + 120\left(\frac{10}{\rho} + 1\right)\max(1, \log(1/\delta))\right)\frac{\sqrt{\VecLOneNorm{L_0}\Delta}}{\sqrt{T}}
    + 20\left(\frac{10}{\rho} + 1\right)\beta_1^{\bar{\tau}}(\VecLOneNorm{M} + \VecLOneNorm{\sigma})\\
    &\le
    \frac{1560\Delta^{1/4}\VecLOneNorm{L_0}^{1/4}\VecLOneNorm{\sigma}^{1/2}\sqrt{\max(1, \log(1/\delta))}}{\rho\sqrt{1-\beta_2}T^{1/4}} + \frac{1330\max(1, \log(1/\delta))\sqrt{\VecLOneNorm{L_0}\Delta}}{\rho\sqrt{T}}\\
    &\quad +\frac{220}{\rho}(\VecLOneNorm{M} + \VecLOneNorm{\sigma})\exp\left(-\frac{\sqrt{1-\beta_2}\VecLOneNorm{L_0}^{3/4}}{\PNormDimension\VecLInftyNorm{L_1}\VecLOneNorm{\sigma}^{1/2}\Delta^{1/4}}T^{1/4}\right)~.
\end{align}

Finally, taking $\delta^{\prime} = \frac{\delta}{3dT}$, we obtain the stated result.
\end{proof}

\clearpage
\subsection{More Experiment Details}
\label{ssec:exp_details}
In this section, we provide more details for the experiments we show in Section~\ref{sec:experiments}.

\textbf{Hyperparameter Tuning} During the validation stage, we used grid-search to fine-tune respective hyperparameters and choose the ones that yield the best validation results.
We tuned the hyperparameters using the following two-stage grid searching strategy: First, search over a coarse grid, and select the one yielding the best validation result. Next, continue searching in a fine grid centering at the best-performing hyperparameters found in the coarse stage, and in turn, take the best one as the final choice. Also, whenever the best-performing hyperparameters lie in the boundary of the searching grid, we always extend the grid to make the final best-performing hyperparameters fall into the interior of the grid, if possible.

\textbf{Resnet on CIFAR-10}
We randomly selected $10\%$ images from the training dataset for validation. Yet, during testing, we trained on the whole training dataset.
The detailed search ranges and the hyperparameter choices yielding the highest validation accuracy for each optimizer are listed in Table~\ref{tab:cifar10_hp}.

\begin{table}[t]
    \centering
    \caption{Hyperparameter grid search ranges and choices yielding the highest validation accuracy for each optimizer for training a 20-layer Resnet to do image classification on CIFAR-10. ({"lr" denotes the initial learning rate, "clip" denotes the clipping parameter $\gamma$ in Algorithm 1 of~\cite{zhang2020improved}, and "$\beta_2$" is defined in~\cite{KingmaB14} for Adam and in Algorithm~\ref{alg:generalized_signsgd} for ours.})}
    \label{tab:cifar10_hp}
    {\small{
    \begin{tabular}{|c|c|c|}
        \hline
        Optimizer & Grid Search Range & Best Choice \\
        \hline
        SGD Momentum & lr \{1e-5, 0.0001, 0.001, 0.01, 0.05, 0.07, 0.1, 0.2, 0.3, 1, 10\} & lr=0.07\\
        \hline
        SGD Momentum Normalized & lr \{0.0001, 0.001, 0.01, 0.05, 0.07, 0.09, 0.1, 0.2, 0.3, 0.5 1, 10\} & lr=0.1\\
        \hline
        SGDClipGrad & \makecell{lr \{0.001, 0.01, 0.05, 0.1, 0.5, 1, 10\}\\ clip \{0.1, 1, 10\}} & \makecell{lr=0.5\\ clip=1} \\
        \hline
        SGDClipMomentum & \makecell{lr \{0.001, 0.01, 0.1, 1, 5, 10, 20, 50\}\\ clip \{0.01, 0.1, 1, 10\}} & \makecell{lr=10\\ clip=0.1} \\
        \hline
        Adam & \makecell{lr \{1e-5, 0.0001, 0.0007, 0.0009, 0.001, 0.002, 0.003, 0.01, 0.1\} \\ $\beta_2$ \{0.4, 0.8, 0.999\}} & \makecell{lr=0.0009\\ $\beta_2$=0.999} \\
        \hline
        Our Algorithm~\ref{alg:generalized_signsgd} & \makecell{lr \{5e-5, 8e-5, 0.0001, 0.0002, 0.0005, 0.001, 0.01\} \\ $\beta_2$ \{0.4, 0.8, 0.999\}} & \makecell{lr=0.0002\\ $\beta_2$ = 0.999} \\
        \hline
    \end{tabular}}
    }
\end{table}

\textbf{AWD-LSTM on Penn Treebank} We used the original train-validation-test split that comes with the dataset. The momentum parameter ($\beta_1)$ is fixed to be $0.9$ except for SGDClipGrad which does not use momentum.
The detailed search ranges and the hyperparameter choices yielding the lowest validation perplexity for each optimizer are listed in Table~\ref{tab:ptb_hp}.

\begin{table}[t]
    \centering
    \caption{Hyperparameter grid search ranges and choices yielding the lowest validation perplexity for each optimizer for training an AWD-LSTM to do language modeling on Penn Treebank. ({"wd" denotes the weight decay value, "lr" denotes the initial learning rate, "clip" denotes the clipping parameter $\gamma$ in Algorithm 1 of~\cite{zhang2020improved}, and "$\beta_2$" is defined in~\cite{KingmaB14} for Adam and in Algorithm~\ref{alg:generalized_signsgd} for ours.})}
    \label{tab:ptb_hp}
    \begin{tabular}{|c|c|c|}
        \hline
        Optimizer & Grid Search Range & Best Choice \\
        \hline
        SGD Momentum & \makecell{wd \{1e-7, 1.2e-6, 5e-6, 1e-5, 1e-4, 1e-3\}\\ lr \{0.001, 0.01, 0.1, 0.5, 0.8, 1, 2, 4, 5\}} & \makecell{wd=1e-5 \\ lr=1}\\
        \hline
        SGD Momentum Normalized & \makecell{wd \{1e-7, 1.2e-6, 5e-6, 1e-5, 5e-5, 1e-4\}\\ lr \{0.01, 0.05, 0.1, 0.5, 0.8, 1, 2, 4, 5, 10\}} & \makecell{wd=5e-6 \\ lr=2}\\
        \hline
        SGDClipGrad & \makecell{wd \{1e-7, 1.2e-6, 5e-6, 1e-5\}\\ lr \{0.1, 0.5, 1, 5, 10, 20, 30, 40, 50, 60, 70\}\\ clip \{1, 2.5, 7.5, 10, 15, 20\}} & \makecell{wd=1.2e-6 \\ lr=50 \\ clip=10}\\
        \hline
        SGDClipMomentum & \makecell{wd \{1e-7, 1.2e-6, 5e-6, 1e-5\}\\ lr \{5, 10, 20, 30, 50, 100\}\\ clip \{1, 2.5, 7.5\}} & \makecell{wd=1.2e-6 \\ lr=20 \\ clip=2.5} \\
        \hline
        Adam & \makecell{wd \{1e-7, 1.2e-6, 5e-6, 1e-5\}\\ lr \{0.0001, 0.001, 0.002, 0.003, 0.01, 0.1\}\\ $\beta_2$ \{0.4, 0.8, 0.999\}} & \makecell{wd=5e-6 \\ lr=0.002 \\ $\beta_2$=0.999} \\
        \hline
        Our Algorithm~\ref{alg:generalized_signsgd} & \makecell{wd \{1e-7, 1.2e-6, 5e-6, 1e-5\}\\ lr \{0.0001, 0.001, 0.002, 0.003, 0.01, 0.1\}\\ $\beta_2$ \{0.4, 0.8, 0.999\}} & \makecell{wd=1.2e-6 \\ lr=0.001 \\ $\beta_2$=0.999}  \\
        \hline
    \end{tabular}
\end{table}

\subsection{Training a Transformer Model on WMT'16 German-English Translation Task}
\label{ssec:transformer_exp}

We noted that Transformers~\cite{VaswaniSPUJGKP17} are gaining huge popularities recently and reported in Figure~\ref{fig:global_l0l1} and~\ref{fig:transformer_l0l1_coordinate_wise} that Transformers observe the relaxed smoothness conditions. Thus, to further showcase the effectivity of our algorithm~\ref{alg:generalized_signsgd} compared with other optimizers listed in~\ref{ssec:comparison}, we train a $6$-layer Transformer model to do machine translation on the WMT'16 Multimodal Machine Translation Task German-English dataset.
The implementation of the transformer is forked from here\footnote{\url{https://github.com/jadore801120/attention-is-all-you-need-pytorch}} and we inherited the default model structure. We also adopted the warm-up steps of $128000$ and the learning rate decay strategy as recommended by the GitHub repo. The mini-batch size is $256$ and we trained for $400$ epochs.

\begin{table}[H]
\centering
\vspace{-1em}
\caption{Average final training loss, test perplexity, and test accuracy achieved by each method when optimizing the Transformer model on the WMT'16 Multimodal Machine Translation Task German-English dataset. The $\pm$ shows $95\%$ confidence intervals of the mean value over 5 runs starting from different random seeds.}
\label{tab:results_multi30k}
\begin{tabular}{|c|c|c|c|}
\hline
Methods
& Training loss & Test perplexity & Test accuracy\\
\hline
SGD Momentum & $2.8045 \pm 0.0209$ & $10.4319 \pm 0.1973$ & $63.9108 \pm 0.5797$\\
\hline
SGD Momentum Normalized & $2.9268 \pm 0.0512$ & $10.5793 \pm 0.6383$ & $63.0774 \pm 0.8231$\\
\hline
SGDClipGrad & $3.0214 \pm 0.0508$ & $10.5974 \pm 0.3527$ & $62.2534 \pm 0.3145$\\
\hline
SGDClipMomentum & $2.8128 \pm 0.0295$ & $10.4677 \pm 0.3619$ & $63.6562 \pm 0.5370$\\
\hline
Adam & $\mathbf{1.4303 \pm 0.0009}$ & $8.9088 \pm 0.1294$ & $\mathbf{68.9828 \pm 0.2786}$\\
\hline
Our Algorithm~\ref{alg:generalized_signsgd} & $1.6263 \pm 0.0024$ & $\mathbf{7.2731 \pm 0.0870}$ & $68.5790 \pm 0.4693$\\
\hline
\end{tabular}
\vspace{-1em}
\end{table}

For the hyperparameter tuning, the momentum parameter ($\beta_1)$ is fixed to be $0.9$ except for SGDClipGrad which does not use momentum. We then used grid-search to fine-tune the initial learning rate and the weight decay value for all optimizers, as well as the clipping threshold for SGDClipGrad and SGDClipMomentum, and $\beta_2$ for Adam and our algorithm. We select the combination of hyperparameters that gives the lowest validation perplexity. The detailed search ranges and the hyperparameter choices yielding the lowest validation perplexity for each optimizer are listed in Table~\ref{tab:ptb_hp}.

\begin{figure}[t]
    \centering
    \includegraphics[width=\textwidth]{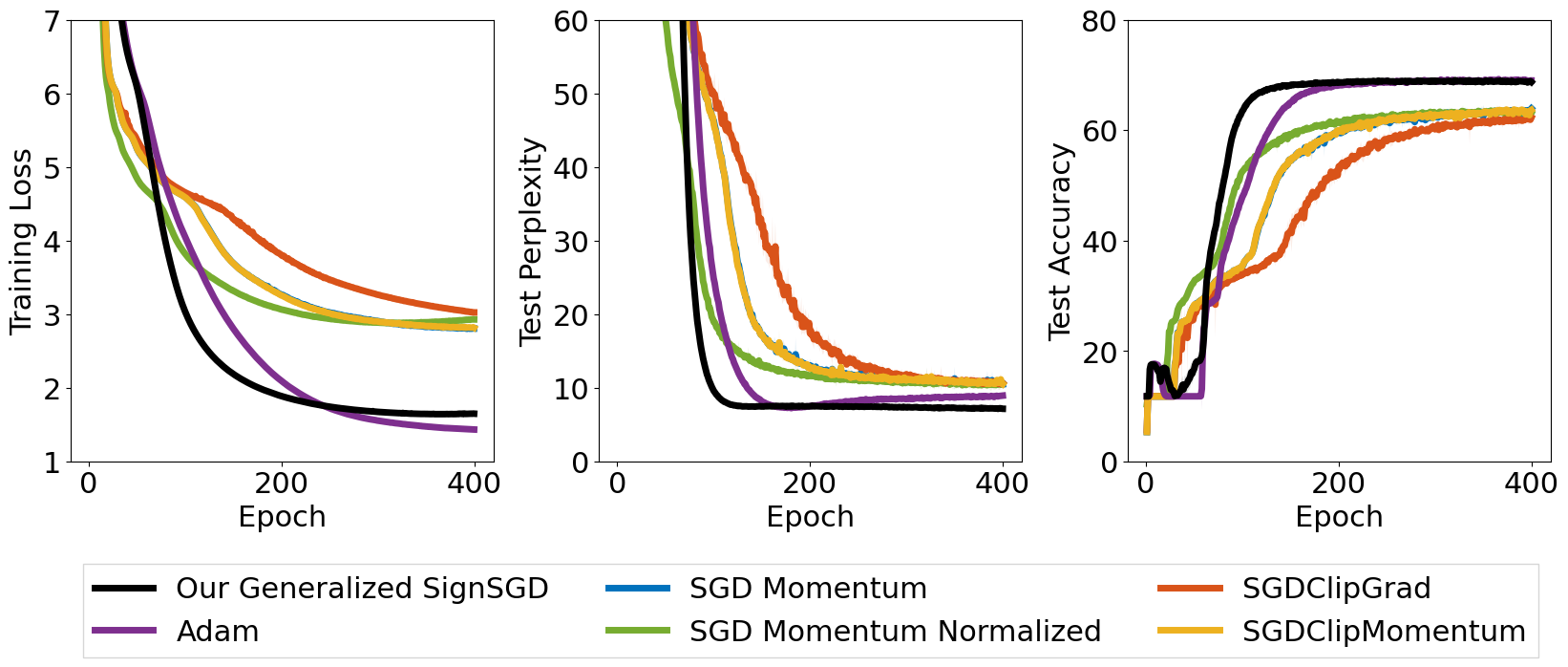}
    \caption{Training a $6$-layer Transformer model to do machine translation on the WMT'16 Multimodal Machine Translation Task German-English dataset. The shading of each curve represents the 95\% confidence interval computed across five independent runs from different random seeds.}
    \label{fig:multi30k}
\end{figure}

\begin{table}[t]
    \centering
    \caption{Hyperparameter grid search ranges and choices yielding the lowest validation perplexity for each optimizer for training a $6$-layer Transformer to do machine translation on the WMT'16 Multimodal Machine Translation Task German-English dataset. ({"wd" denotes the weight decay value, "lr" denotes the initial learning rate, "clip" denotes the clipping parameter $\gamma$ in Algorithm 1 of~\cite{zhang2020improved}, and "$\beta_2$" is defined in~\cite{KingmaB14} for Adam and in Algorithm~\ref{alg:generalized_signsgd} for ours.})}
    \label{tab:multi30k_hp}
    \begin{tabular}{|c|c|c|}
        \hline
        Optimizer & Grid Search Range & Best Choice \\
        \hline

        SGD Momentum & \makecell{wd \{1e-6, 0.001, 0.01, 0.1, 1, 10, 100, 1000\}\\ lr \{0.1, 1, 10, 100, 1000, 10000\}} & \makecell{wd=1 \\ lr=1}\\
        \hline

        SGD Momentum Normalized & \makecell{wd \{1e-6, 0.001, 0.01, 0.1, 1, 10, 100, 1000\}\\ lr \{0.01, 0.1, 1, 10, 100, 1000, 10000, 100000\}} & \makecell{wd=1 \\ lr=10000}\\
        \hline
        SGDClipGrad & \makecell{wd \{1e-6, 0.01, 0.1, 1, 10, 100, 1000\}\\ lr \{0.01, 0.1, 1, 10, 100, 1000, 10000\}\\ clip \{0.01, 0.1, 1, 10\}} & \makecell{wd=1 \\ lr=10 \\ clip=1}\\
        \hline
        SGDClipMomentum & \makecell{wd \{1e-6, 0.01, 0.1, 1, 10\}\\ lr \{0.01, 0.1, 1, 10, 100, 1000, 10000\}\\ clip \{0.01, 0.1, 1, 10\}} & \makecell{wd=1 \\ lr=1 \\ clip=1} \\
        \hline
        Adam & \makecell{wd \{1e-6, 0.001, 0.01, 0.1, 1, 10, 100\}\\ lr \{0.1, 1, 10, 100, 1000, 10000\}\\ $\beta_2$ \{0.9, 0.98, 0.999\}} & \makecell{wd=0.1 \\ lr=10 \\ $\beta_2$=0.98} \\
        \hline
        Our Algorithm~\ref{alg:generalized_signsgd} & \makecell{wd \{1e-6, 0.001, 0.01, 0.1, 1, 10, 100\}\\ lr \{0.1, 1, 10, 100, 1000, 10000\}\\ $\beta_2$ \{0.9, 0.98, 0.999\}} & \makecell{wd=0.1 \\ lr=10 \\ $\beta_2$=0.98}  \\
        \hline
    \end{tabular}
\end{table}

We then employ the best-performing hyperparameters to report the testing performance. The testing is repeated with random seeds 5 times to eliminate the influence of stochasticity. The results are reported in Figure~\ref{fig:multi30k} and Table~\ref{tab:results_multi30k}. Here the accuracy means the proportion of correct correspondences of words, namely the same word at the same location, between the machine-translated output and the target. It can be seen that we can still match the performance of Adam while beating the others. Also, note that the curves of SGD Momentum and the ones of SGDClipMomentum overlap as they utilize the same weight decay values and initial learning rates, and turns out clipping is seldomly performed when employing SGDClipMomentum.

\end{document}